\documentclass{article}

\usepackage{microtype}
\usepackage[disable,textsize=tiny]{todonotes}
\usepackage{graphicx}
\usepackage{subcaption}
\usepackage{booktabs} 

\usepackage{hyperref}

\usepackage{natbib}
\usepackage[preprint]{icml2026}

\usepackage{amsmath}
\usepackage{amssymb}
\usepackage{mathtools}
\usepackage{amsthm}
\newtheorem{theorem}{Theorem}[section]

\newtheorem{lemma}[theorem]{Lemma}

\newtheorem{definition}[theorem]{Definition}
\newtheorem{assumption}[theorem]{Assumption}
\newtheorem{remark}[theorem]{Remark}

\def\*#1{\boldsymbol{#1}}

\newcommand{\signm}{\mathrm{sign}}  
\newcommand{\srank}{\mathrm{srank}} 

\newcommand{\mA}{\mathbf{A}}
\newcommand{\mB}{\mathbf{B}}
\newcommand{\mQ}{\mathbf{Q}}
\newcommand{\mK}{\mathbf{K}}
\newcommand{\mW}{\mathbf{W}}
\newcommand{\mU}{\mathbf{U}}
\newcommand{\mV}{\mathbf{V}}
\newcommand{\mS}{\mathbf{S}}
\newcommand{\mH}{\mathbf{H}}
\newcommand{\mX}{\mathbf{X}}
\newcommand{\mY}{\mathbf{Y}}

\newcommand{\mJ}{\mathbf{J}}

\newcommand{\va}{\mathbf{a}}

\newcommand{\vh}{\mathbf{h}}

\newcommand{\vs}{\mathbf{s}}
\newcommand{\vy}{\mathbf{y}}
\newcommand{\vx}{\mathbf{x}}
\newcommand{\vv}{\mathbf{v}}
\newcommand{\vz}{\mathbf{z}}
\newcommand{\vu}{\mathbf{u}}

\newcommand{\R}{\mathbb{R}}

\newcommand{\E}{\mathbb{E}}

\usepackage{enumitem}
\setlist{topsep=0pt, partopsep=0pt, itemsep=0pt, parsep=2pt}
\makeatletter
\renewcommand\paragraph{\@startsection{paragraph}{4}{\z@}%
  {0pt}%
  {-1em}%
  {\normalfont\normalsize\bfseries}}
\makeatother

\usepackage[capitalize,noabbrev]{cleveref}

\theoremstyle{plain}
\theoremstyle{definition}
\theoremstyle{remark}

\icmltitlerunning{MSign: An Optimizer Preventing Training Instability in Large Language Models via Stable Rank Restoration}

\begin{document}

\twocolumn[
  \icmltitle{
  MSign: An Optimizer Preventing Training Instability in Large Language Models via Stable Rank Restoration 
  }



  \icmlsetsymbol{equal}{*}
  \icmlsetsymbol{intern}{$\dagger$}

  \begin{icmlauthorlist}
    \icmlauthor{Lianhai Ren}{comp,nus,intern}
    \icmlauthor{Yucheng Ding}{comp,intern}
    \icmlauthor{Xiao Liu}{comp}
    \icmlauthor{Peng Cheng}{comp}
    \icmlauthor{Yeyun Gong}{comp}
  \end{icmlauthorlist}

  \icmlaffiliation{nus}{Department of Mathematics, National University of Singapore, 10 Lower Kent Ridge Rd, Singapore 119076}
  \icmlaffiliation{comp}{Microsoft SIGMA Team, Microsoft Research.}

  \icmlcorrespondingauthor{Peng Cheng}{pengc@microsoft.com}
  \icmlcorrespondingauthor{Yeyun Gong}{yegong@microsoft.com}

  \icmlkeywords{Machine Learning}

  \vskip 0.3in
]



\printAffiliationsAndNotice{$^\dagger$Work done during their internship at Microsoft Research.}

\begin{abstract}
Training instability remains a critical challenge in large language model (LLM) pretraining, often manifesting as sudden gradient explosions that waste significant computational resources.
We study training failures in a 5M-parameter NanoGPT model scaled via $\mu$P, identifying two key phenomena preceding collapse: (1) rapid decline in weight matrix stable rank (ratio of squared Frobenius norm to squared spectral norm), and (2) increasing alignment between adjacent layer Jacobians.
We prove theoretically that these two conditions jointly cause exponential gradient norm growth with network depth.
To break this instability mechanism, we propose MSign, a new optimizer that periodically applies matrix sign operations to restore stable rank.
Experiments on models from 5M to 3B parameters demonstrate that MSign effectively prevents training failures with a computational overhead of less than 7.0\%.
\end{abstract}

\section{Introduction}

Modern deep learning optimizers and training techniques -- including Adam~\citep{kingma2014adam}, layer normalization~\citep{ba2016layer}, and learning rate schedules -- generally enable reliable training. However, as the model scale increases, pretraining large language models (LLMs) becomes increasingly fragile. Training failures manifest as sudden, unrecoverable gradient explosions and corresponding loss growth, which are difficult to predict and can waste substantial computational resources~\citep{chowdhery2022palm,zhang2022opt}.
\paragraph{Empirical Investigation: Identifying the Failure Mechanism.}
We systematically study training failures using a 5M-parameter NanoGPT model derived from $\mu$P scaling~\citep{yang2022tensor}, which provides a controlled environment for reproducible failure analysis. Through extensive monitoring, we identify two critical phenomena that consistently precede training collapse (Figure~\ref{fig:weight-scale}):

\begin{itemize}
\item \textbf{Observation 1: Stable Rank Collapse.} The stable rank of weight matrices -- defined as $\srank(\mW) = \|\mW\|_F^2 / \|\mW\|_2^2$ -- declines sharply in the steps preceding failure. This indicates that spectral energy becomes concentrated in the top singular directions, reducing the ``effective dimensionality'' of the weight matrices.

\item \textbf{Observation 2: Jacobian Alignment Growth.} The alignment between adjacent layer Jacobians increases, meaning the top singular subspaces of consecutive layers become increasingly correlated. It prevents the typical cancellation effects in matrix products.
\end{itemize}

We prove that these two phenomena jointly cause training instability: low stable rank implies high layer Jacobian spectral norms (since $\|\mW\|_2 = \|\mW\|_F / \sqrt{\srank(\mW)}$ for fixed Frobenius norm), and high alignment ensures these norms multiply constructively across layers. The total Jacobian norm grows as $(aM)^L$ where $a$ is alignment, $M$ is layer Jacobian norm, and $L$ is depth, leading to exponential gradient explosion when $aM > 1$.

\paragraph{Our Solution: The MSign Optimizer.}
To break the stable rank collapse condition, we propose \textbf{MSign}, an optimizer that periodically applies the matrix sign operation $\signm(\mW) = \mU\mV^T$ (where $\mW = \mU\mS\mV^T$ is the SVD). This operation maximizes stable rank by equalizing all non-zero singular values to 1, while preserving the column and row spaces. We restore the original Frobenius norm after the operation to maintain training dynamics. MSign is applied every $P$ steps (typically $P=100$) to projection weights, with computational overhead below 7.0\%.

\paragraph{Experimental Validation.}
We validate MSign across four model configurations spanning 5M to 3B parameters: NanoGPT-5M~\footnote{\href{https://github.com/karpathy/nanoGPT}{https://github.com/karpathy/nanoGPT}}, Sigma-40M (hybrid MHA/MLA attention)~\cite{qu2025sigma,hu2025sigmatiny}, LLaMA-1B~\cite{kumar2025zclip}, and LLaMA-MoE-3B (mixture of experts). In all cases, baseline training with standard hyperparameters fails via gradient explosion, while training with MSign converges stably. The intervention maintains stable rank above critical thresholds, controls Jacobian alignment, and keeps gradient norms bounded. Ablation studies reveal that applying MSign to attention layers (particularly output projections) is sufficient, MLP-only application does not prevent failures.

\paragraph{Contributions.}
We summarize our contributions:
\begin{enumerate}
\item \textbf{Mechanism identification and theoretical analysis}: We identify stable rank collapse and Jacobian alignment as consistent precursors to LLM training failures, and prove that their combination causes exponential gradient growth with network depth.

\item \textbf{Practical solution}: We propose MSign, a new optimizer that prevents training failures by periodically restoring weight stable rank.

\item \textbf{Extensive validation}: We demonstrate MSign's effectiveness across diverse architectures (dense, MoE) and scales (5M--3B parameters) with minimal overhead.
\end{enumerate}
\section{Literature Review}
\paragraph{Training Instability in Large Language Models.}
Training instability in large language models has been widely observed across major LLM projects. \citet{chowdhery2022palm} documented loss spikes during PaLM training that required manual intervention and checkpoint rollbacks, while \citet{zhang2022opt} provided detailed logs of OPT-175B training showing dozens of restarts due to hardware failures and gradient explosions. \citet{zeng2022glm} reported similar challenges in GLM-130B training. Several factors have been proposed to explain instability: \citet{kaplan2020scaling} studied the sensitivity of training dynamics to learning rate in the context of scaling laws; \citet{pascanu2013difficulty} analyzed gradient clipping as a remedy for exploding gradients in RNNs; \citet{dong2021attention} identified attention entropy collapse as a failure mode in Vision Transformers. More recently, \citet{wortsman2023small} showed that small-scale proxy models can predict instabilities at larger scales, and \citet{molybog2023theory} proposed a theoretical framework connecting loss spikes to heavy-tailed gradient noise. However, these explanations typically address symptoms rather than underlying mechanisms.
\paragraph{Low-Rank Structure in Neural Networks.}
The prevalence of low-rank structure in neural network weights and gradients has been extensively documented. \citet{denil2013predicting} demonstrated that neural network weights exhibit significant redundancy, with up to 95\% of parameters predictable from the remaining 5\%. \citet{arora2019exact} proved that gradient descent in deep linear networks implicitly biases toward low-rank solutions. In the transformer context, \citet{hu2022lora} exploited low-rank structure for parameter-efficient fine-tuning, while \citet{galore2023} showed that gradients during transformer training maintain consistently low stable rank, motivating memory-efficient optimizers. \citet{huh2024platonic} found that representations across different models and modalities converge to similar low-dimensional structures. Our work extends these observations by connecting low-rank gradient structure directly to training instability through the stable rank mechanism.
\paragraph{Jacobian Analysis in Deep Networks.}
The role of Jacobians in neural network optimization has received substantial theoretical attention. Classical work by \citet{glorot2010understanding} established the importance of proper weight initialization for maintaining stable gradient flow. \citet{saxe2014exact} analyzed dynamics in deep linear networks through the lens of Jacobian singular values. \citet{pennington2017resnet} and \citet{yang2017mean} developed mean-field theories for Jacobian evolution in residual networks and general deep networks respectively. \citet{fort2019deep} empirically studied the relationship between Jacobian eigenvalue spectra and model trainability, finding that successful training correlates with specific spectral properties. \citet{xiao2018dynamical} proposed dynamical isometry as a condition for stable training. Our contribution is identifying inter-layer Jacobian \emph{alignment} as a key mechanism amplifying gradient explosion, distinct from prior work that focused on individual layer Jacobian norms.
\paragraph{Stable Rank and Matrix Analysis.}
Stable rank, introduced by \citet{rudelson2007sampling} in the context of random matrix theory, provides a continuous relaxation of matrix rank that is more robust to small perturbations. \citet{vershynin2018high} provides comprehensive treatment of stable rank properties and its applications in high-dimensional probability. In neural network analysis, stable rank has been used to derive generalization bounds: \citet{NeyshaburBMS17aa} used stable rank to obtain PAC-Bayes bounds. \citet{li2018measuring} analyzed intrinsic rank dynamics during optimization, finding that it correlates with model compressibility. \citet{sanyal2020stable} proposed stable rank regularization to improve generalization. Our work reveals a novel connection between stable rank dynamics and training stability, showing that stable rank collapse precedes and causes gradient explosion.
\begin{figure}[t]
    \begin{center}
    \includegraphics[width=0.7\linewidth]{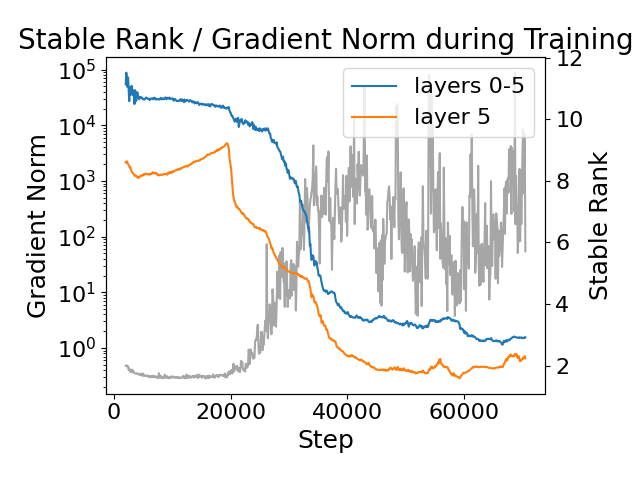}
    \includegraphics[width=0.7\linewidth]{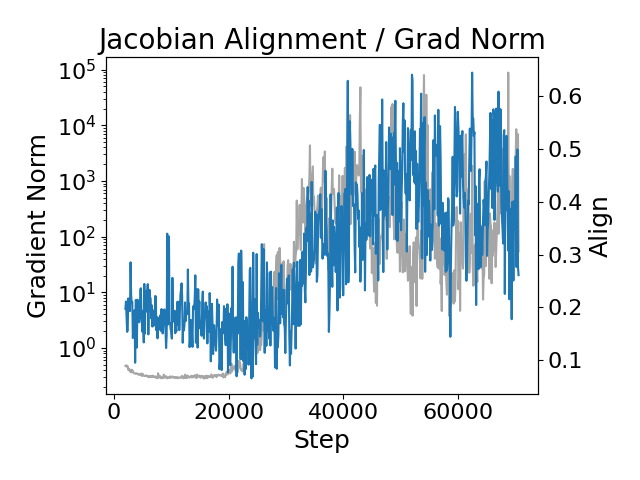}
    \caption{Correlation between training failure indicators and gradient norm explosion. \textbf{Left (Observation 1):} Stable rank (geometric mean across early layers) vs.\ gradient norm over training steps. As stable rank declines sharply around step 20000, gradient norms begin explosive growth. \textbf{Right  (Observation 2):} Jacobian alignment (average between adjacent layers) vs.\ gradient norm. Increasing alignment precedes and accompanies gradient explosion.}
    \label{fig:weight-scale}
    \end{center}
    \vspace{-10pt}
\end{figure}
\section{Empirical Observations: Training Failure Phenomena}
We identify reproducible training failures in transformer models under moderate learning rates. Our analysis reveals two consistent phenomena preceding training collapse.

\subsection{Experimental Setup}
We construct a reproducible failure scenario using a modified NanoGPT configuration with standard hyperparameters (refer to Appendix~\ref{sec:exp-details} for details).

From $\mu$P perspective~\cite{yang2022tensor}, this configuration corresponds to a $0.02$ std and $6\times 10^{-4}$ learning rate at $100$M scale, which is within typical ranges.
\paragraph{Model and Jacobian Notation.}
We consider a standard decoder-only transformer with $L$ stacked blocks. At each block $\ell$, hidden states $\mH^{(\ell-1)} \in \R^{T \times d}$ (sequence length $T$, hidden dimension $d$) are first processed by multi-head self-attention with query/key/value/output projections $\mW_Q^{(\ell)}, \mW_K^{(\ell)}, \mW_V^{(\ell)}, \mW_O^{(\ell)}$, and then by a position-wise MLP with weights $\mW_1^{(\ell)} \in \R^{d \times d_{\mathrm{ff}}}$ and $\mW_2^{(\ell)} \in \R^{d_{\mathrm{ff}} \times d}$, with residual connections and LayerNorm around each sublayer:
\begin{equation}
    \mH^{(\ell)} = F^{(\ell)}(\mH^{(\ell-1)}), \quad \ell = 1, \dots, L.
\end{equation}
We define the \emph{layer Jacobian} as $\mJ^{(\ell)} = \frac{\partial \, \mathrm{vec}(\mH^{(\ell)})}{\partial \, \mathrm{vec}(\mH^{(\ell-1)})}$. Later references to the stable ranks of $\mW_Q, \mW_K, \mW_V, \mW_O, \mW_1, \mW_2$ and to "Jacobian alignment between adjacent layers" all refer to this architecture and these $\mJ^{(\ell)}$.
\subsection{Observations during training failure}

\paragraph{Observation 1: Sharp Stable Rank Decline} The first observable phenomenon is the rapid decline in weight stable rank. As shown in Figure~\ref{fig:weight-scale} (left panel), the geometric averaged stable rank of the first several layers drops sharply around step 20000, preceding the gradient explosion. This phenomenon indicates energy concentration in top singular values.

\begin{definition}[Stable Rank]
The stable rank of a matrix $\mW \in \mathbb{R}^{m \times n}$ is:
\begin{equation}
\srank(\mW) = \frac{\|\mW\|_F^2}{\|\mW\|_2^2} = \frac{\sum_i s_i^2}{s_1^2},
\end{equation}
where $s_1 \geq s_2 \geq \cdots \geq 0$ are the singular values.
\end{definition}

\paragraph{Observation 2: Increasing Jacobian Alignment Between Adjacent Layers}
The second critical phenomenon is the increasing alignment between Jacobians of adjacent layers, as illustrated in Figure~\ref{fig:weight-scale} (right panel).
\begin{definition}[Matrix Alignment]
For any two matrices $\mA, \mB$ for which the product $\mA\mB$ is well-defined, their alignment is the cosine similarity between the top right singular vector of $\mA$ and the top left singular vector of $\mB$:
\begin{equation}
\text{Align}(\mA, \mB) = |\vv_{A,1}^T \vu_{B,1}|,
\end{equation}
where $\vv_{A,1}$ is the first right singular vector of $\mA$ and $\vu_{B,1}$ is the first left singular vector of $\mB$. For simplicity, we denote the alignment of Jacobians for adjacent layers, $\text{Align}(\mJ^{(\ell+1)}, \mJ^{(\ell)})$, as $\text{Align}(\ell+1, \ell)$.
\end{definition}

High alignment correlates strongly with both weight scale growth and stable rank decline, as well as gradient growth during the failure phase.

\subsection{Conjecture: Low Stable Rank + Jacobian Alignment Drives Training Failure}

Based on these two consistent phenomena, we conjecture that the combination of low weight stable rank and high Jacobian alignment creates a destabilizing feedback mechanism that leads to training failure. The remainder of this paper develops this hypothesis theoretically and proposes a solution.

\section{Theoretical Analysis: Understanding the Failure Mechanism}

We now provide theoretical analysis to explain the observed phenomena and their causal relationships. Our analysis is divided into two parts: (1) explaining why Jacobian alignment and low stable rank lead to training failure, and (2) analyzing the positive feedback mechanism that prevents stable rank from increasing and accelerates its collapse.

\subsection{Part I: From Observations to Training Failure}

In this section, we establish the causal chain: \textbf{Low stable rank + Jacobian alignment} $\Rightarrow$ \textbf{High total Jacobian norm} $\Rightarrow$ \textbf{Large weight gradient norm} (\textbf{Training instability}).

\begin{remark}[Simplifying Assumption]
We adopt the standard assumption that large gradient norms indicate training instability~\citep{pascanu2013difficulty,philipp2018exploding}. Our analysis derives lower bounds for gradient norms, with the understanding that such bounds indicate increased risk of divergence.
\end{remark}

\subsubsection{High Jacobian Alignment + High Layer Jacobian Norm $\Rightarrow$ High Total Jacobian Norm}

The total Jacobian $\mJ_{total} = \prod_{\ell=1}^L \mJ^{(\ell)}$ determines how perturbations at the input propagate to the output. In general, the norm of a matrix product can be much smaller than the product of norms due to cancellation effects. However, when the singular subspaces of adjacent Jacobians are \emph{aligned}, these cancellations are suppressed, and the norms multiply constructively.

To formalize this, we use the definition of matrix alignment (Definition 3.2).
This quantity measures how well the "output direction" of $\mJ^{(\ell)}$ matches the "input direction" of $\mJ^{(\ell+1)}$, determining whether the composition $\mJ^{(\ell+1)}\mJ^{(\ell)}$ preserves or diminishes the spectral norm.

\begin{theorem}[Jacobian Product Norm Lower Bound]
\label{thm:jacobian-product}
For a deep network with $L$ layers, let $\mJ^{(\ell)} = \frac{\partial \vh^{(\ell)}}{\partial \vh^{(\ell-1)}}$ denote the Jacobian at layer $\ell$. If each layer Jacobian satisfies $\|\mJ^{(\ell)}\|_2 \geq M$ and the alignment between adjacent Jacobians satisfies $\text{Align}(\ell+1, \ell) \geq a > 0$ for all $\ell$, then the total Jacobian from input to output has 2-norm:
\begin{equation}
\|\mJ_{total}\|_2 = \|\mJ^{(L)} \mJ^{(L-1)} \cdots \mJ^{(1)}\|_2 \geq \frac{(aM)^L}{a}.
\end{equation}
\end{theorem}

The proof is provided in Appendix~\ref{proof:jacobian-product}.

\begin{remark}[Exponential Growth Condition]
The key insight is that when $aM > 1$, the lower bound in Theorem~\ref{thm:jacobian-product} grows \emph{exponentially} with depth $L$, providing a sufficient condition for large total Jacobian norms. Observation 1 (stable rank collapse) suggests that $M$ can become large (via Theorem~\ref{thm:srank-jacobian} under approximately fixed Frobenius norms), and Observation 2 shows that $a$ tends to increase during training. In the failure regimes we study, these trends empirically drive $aM$ above $1$, which is consistent with the observed gradient explosion, although we do not claim that $aM>1$ holds at all times or is necessary for failure.
\end{remark}

\subsubsection{Low Stable Rank $\Rightarrow$ High Layer Jacobian Norm}

We analyze the relationship between stable rank and layer Jacobian norm for the three primary layer types in transformers.

\paragraph{Linear Layers.}

\begin{theorem}[Stable Rank Controls Jacobian Norm: Linear Layer]
\label{thm:srank-jacobian}
For a linear layer with weight matrix $\mW \in \R^{m \times n}$, given fixed Frobenius norm $\|\mW\|_F = F$, the operator norm satisfies:
\begin{equation}
\|\mW\|_2 = \frac{F}{\sqrt{\srank(\mW)}}.
\end{equation}
\end{theorem}

The proof is provided in Appendix~\ref{proof:srank-jacobian}. This establishes the basic principle: as stable rank decreases while the Frobenius norm is kept fixed (or approximately fixed over short training windows, e.g., under Adam-like updates with bounded step sizes), the operator norm increases proportionally.

\paragraph{Attention Layers.}

\begin{theorem}[Jacobian Norm Bound: Attention Layer]
    \label{thm:jacobian-attention}
Consider a single-head attention layer with projections $\mW_Q, \mW_K, \mW_V, \mW_O \in \R^{d \times d_k}$. Let $\mH \in \R^{n \times d}$ be the input, $\mA = \text{softmax}\left(\frac{\mH\mW_Q \mW_K^T\mH^T}{\sqrt{d_k}}\right)$ be the attention matrix. The Jacobian of the attention output $\mY = \mA \mH\mW_V \mW_O$ satisfies:
\begin{equation}
\begin{aligned}
    \left\|\frac{\partial \mY}{\partial \mH}\right\|_2 \leq& \|\mA\|_2 \|\mW_V\|_2 \|\mW_O\|_2\\    &+ \left\|\frac{\partial \mA}{\partial \mH}\right\|_2\|\mH\|_2 \|\mW_V\|_2 \|\mW_O\|_2.
\end{aligned}
\end{equation}
Defining the logit margin $\gamma_{\min} = \min_i (\max_j \mS_{i,j} - \text{second\_max}_j \mS_{i,j})$ where $\mS = \mH \mW_Q \mW_K^T \mH^T$, the attention gradient pathway is bounded by:
\begin{equation}
\left\|\frac{\partial \mA}{\partial \mH}\right\|_2 \leq \frac{4 \min ((n-1)e^{-\gamma_{\min}},1)}{\sqrt{d_k}} \|\mW_Q\|_2 \|\mW_K\|_2.
\end{equation}
\end{theorem}

The proof is provided in Appendix~\ref{proof:jacobian-attention}. Substituting $\|\mW_i\|_2 = \|\mW_i\|_F / \sqrt{\srank(\mW_i)}$ shows that low stable rank in any projection matrix amplifies the Jacobian norm.

\paragraph{MLP Layers.}

\begin{theorem}[Jacobian Norm Bound: MLP Layer]
\label{thm:jacobian-mlp}
Consider a two-layer MLP with weights $\mW_1 \in \R^{d \times d_{\text{ff}}}$, $\mW_2 \in \R^{d_{\text{ff}} \times d}$, and activation $\phi$. The Jacobian satisfies:
\begin{equation}
\left\|\mJ_{\text{MLP}}\right\|_2 \leq \frac{L_\phi \|\mW_1\|_F \|\mW_2\|_F}{\sqrt{\srank(\mW_1) \cdot \srank(\mW_2)}},
\end{equation}
where $L_\phi$ is the Lipschitz constant of $\phi$ ($L_\phi \approx 1.13$ for GELU, $L_\phi \approx 1.1$ for SiLU).
\end{theorem}

The proof is provided in Appendix~\ref{proof:jacobian-mlp}. Across all layer types, the Jacobian norm is inversely related to the square root of stable rank, creating conditions for exponential gradient growth when combined with Jacobian alignment (Theorem~\ref{thm:jacobian-product}).

\subsubsection{High Total Jacobian Norm $\Rightarrow$ Large Weight Gradient Norm}

We now show that high total Jacobian norms translate into large weight gradients. By the chain rule, the gradient with respect to weight vector $\hat{\vv}_{out}^{(i)}$ at layer $i$ decomposes as:
\begin{equation}
\frac{\partial L}{\partial \hat{\vv}_{out}^{(i)}} = \left(\frac{\partial \vh^{(i)}}{\partial \hat{\vv}_{out}^{(i)}}\right)^T \left(\frac{\partial \vh^{(L)}}{\partial \vh^{(i)}}\right)^T \frac{\partial L}{\partial \vh^{(L)}}.
\label{eq:grad-decomp}
\end{equation}

\begin{assumption}[Gradient Alignment Conditions (Informal)]
\label{assum:grad-align}
We assume: (1) uniform local gradient lower bound $\gamma > 0$; (2) local-Jacobian alignment: $\frac{\partial \vh^{(i)}}{\partial \hat{\vv}_{out}^{(i)}}$ is aligned with the top right singular direction of the \emph{local} layer Jacobian $\mJ^{(i+1)}$, with alignment $\geq a$; (3) terminal alignment: the last layer Jacobian $\mJ^{(L)}$'s top left singular direction is aligned with the loss gradient $\frac{\partial L}{\partial \vh^{(L)}}$, with alignment $\geq a$. The formal statement is provided in Assumption~\ref{assum:grad-align-formal} in the Appendix.
\end{assumption}

\begin{theorem}[Weight Gradient Norm Lower Bound]
\label{thm:jacobian-to-gradient}
Under Assumption~\ref{assum:grad-align}, combined with Theorem~\ref{thm:jacobian-product} when $\|\mJ^{(\ell)}\|_2 \geq M>1$ and alignment $\geq a$:
\begin{equation}
\left\|\frac{\partial L}{\partial \hat{\vv}_{out}^{(i)}}\right\|_2 \geq a \gamma (aM)^{L-i} \cdot \left\|\frac{\partial L}{\partial \vh^{(L)}}\right\|_2.
\end{equation}
\end{theorem}

The proof is provided in Appendix~\ref{proof:jacobian-to-gradient}. Figure~\ref{fig:jacobian-lb} validates this bound empirically.

\begin{figure}
    \centering
    \includegraphics[width=0.7\linewidth]{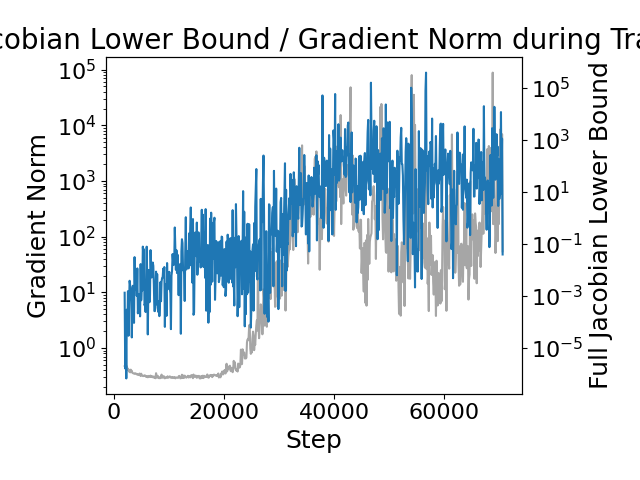}
    \caption{Validation of Theorem~\ref{thm:jacobian-product}: Jacobian product norm lower bound vs.\ actual gradient norm. The theoretical bound closely tracks observed gradient growth.}
    \label{fig:jacobian-lb}
    \vspace{-10pt}
\end{figure}

\begin{theorem}[Total Gradient Norm Lower Bound]
\label{thm:total-gradient}
Summing over all layers:
\begin{equation}
\sum_{i=1}^{L} \left\|\frac{\partial L}{\partial \mW^{(i)}}\right\|_F^2 \geq C \cdot \frac{(aM)^{2L} - 1}{(aM)^2 - 1},
\end{equation}
where $C = a^2 \gamma^2 n_w \left\|\frac{\partial L}{\partial \vh^{(L)}}\right\|_2^2$. When $aM > 1$, this exhibits \textbf{exponential growth} in depth.
\end{theorem}

The proof is provided in Appendix~\ref{proof:total-gradient}. From Observations 1 and 2, stable rank collapse and Jacobian alignment together cause $aM > 1$, triggering exponential gradient explosion.

\subsubsection{Summary: The Failure Pathway}

Combining the above results, we establish the complete causal chain:

\begin{enumerate}
\item Low stable rank (\textbf{Observation 1}) $\Rightarrow$ High layer Jacobian 2-norm (Theorems~\ref{thm:srank-jacobian}, \ref{thm:jacobian-attention}, \ref{thm:jacobian-mlp}) under approximately fixed Frobenius norms.
\item Jacobian alignment (\textbf{Observation 2}) + High layer Jacobian norm $\Rightarrow$ High total Jacobian norm (Theorem~\ref{thm:jacobian-product}).
\item High total Jacobian norm + Gradient alignment (Assumption~\ref{assum:grad-align}) $\Rightarrow$ Large weight gradient (training instability) (Theorems~\ref{thm:jacobian-to-gradient}, \ref{thm:total-gradient}).
\end{enumerate}

Taken together, these ingredients provide a \emph{sufficient} mechanistic explanation for how low stable rank combined with Jacobian alignment can lead to training failure in the regimes we study; we do not claim that this mechanism is necessary for all possible failure modes.

\subsection{Part II: The Positive Feedback Mechanism}

Having established why low stable rank and Jacobian alignment lead to training failure, we now analyze why these conditions tend to \emph{intensify} during training.

\subsubsection{Low-Rank Hidden States Lead to Low-Rank Gradients}
\label{sec:theory-feedback}

We first establish that low-rank structure propagates through the network and affects gradient structure.

\begin{theorem}[Low-Rank Propagation in Attention Layers]
\label{thm:lowrank-propagation}
Consider an attention layer with query, key, value, and output projection matrices $\mW_Q, \mW_K, \mW_V, \mW_O$. If the hidden states $\mH^{(\ell-1)}$ and cohidden states $\tilde{\mH}^{(\ell)}$ have rank at most $r$, then the gradients $\nabla_{\mW_Q} L, \nabla_{\mW_K} L, \nabla_{\mW_V} L, \nabla_{\mW_O} L$ all have rank at most $r$.
\end{theorem}

The proof is provided in Appendix~\ref{proof:lowrank-propagation}.

\begin{remark}
This result extends to MLP layers due to their two-layer structure. For MLP up-projection $\mW_1$: $\nabla_{\mW_1} L = \mH^T \cdot (\text{something})$, so rank is bounded by $\text{rank}(\mH)$. The key insight is that outer-product gradients inherit the rank of their lower-rank factor.
\end{remark}

\subsubsection{Aligned Input-Weight-Output Structure Accelerates Stable Rank Decline}

In general, characterizing how gradient descent affects stable rank dynamics is challenging, as the update direction depends on the complex interplay between input statistics, weight structure, and backpropagated gradients. Here we provide a simplified analysis in a stylized linear setting under strong alignment assumptions that captures the essential feedback mechanism. While these assumptions are restrictive and do not aim to exactly describe full transformer dynamics, they offer theoretical insight into why stable rank can tend to decline during training.
\begin{theorem}[Stable Rank Feedback Mechanism]
\label{thm:srank-feedback}
Consider a linear layer with weight matrix $\mW = \mU \mS \mV^T$ where the input hidden states and output cohidden states satisfy:
\begin{itemize}
\item Input covariance: $\Sigma_{in} = \E[\vh_{in} \vh_{in}^T] = \mV \Lambda_{in} \mV^T$ (aligned with $\mW$'s right singular vectors)
\item Output gradient covariance: $\Sigma_{out} = \E[\tilde{\vh}_{out} \tilde{\vh}_{out}^T] = \mU \Lambda_{out} \mU^T$ (aligned with $\mW$'s left singular vectors)
\end{itemize}
If the correlation between input and output gradient projections $\text{Cov}(\vu_i^T \tilde{\vh}_{out}, \vv_i^T \vh_{in})$ is negative and satisfy
\begin{equation}
    \frac{\text{Cov}(\vu_1^T \tilde{\vh}_{out}, \vv_1^T \vh_{in})}{\text{Cov}(\vu_i^T \tilde{\vh}_{out}, \vv_i^T \vh_{in})}>\frac{s_1}{s_i},\forall 1<i\leq n
\end{equation}
then gradient descent causes the stable rank of $\mW$ to \emph{decrease}.
\end{theorem}
The proof is provided in Appendix~\ref{proof:srank-feedback}. This theorem should be interpreted as an existence result in a highly aligned regime that illustrates one concrete way in which gradient descent can decrease stable rank, rather than as a statement about typical training trajectories.




\begin{algorithm}[tb]
\caption{MSign Optimizer}
\label{alg:msign}
\begin{algorithmic}
    \STATE {\bfseries Input:} parameters $\theta$, gradients $g$, learning rate $\eta$, period $P$, step $t$, target layers
    \STATE {\bfseries Procedure} StepWithMSign:
    \STATE $\theta \leftarrow \text{BaseOptimizerStep}(\theta, g, \eta)$ \hfill $\triangleright$ Standard optimizer update
    \IF{$t \bmod P == 0$}
        \FOR{each parameter $\mW$ in target layers}
            \IF{$\mW$.ndim $\geq 2$}
                \STATE $F \leftarrow \|\mW\|_F$ \hfill $\triangleright$ Record Frobenius norm
                \STATE $\mU, \mS, \mV^T \leftarrow \text{SVD}(\mW)$
                \STATE $\mW \leftarrow \frac{F}{\|\mU \mV^T\|_F} \mU \mV^T$ \hfill $\triangleright$ Matrix sign with norm restoration
            \ENDIF
        \ENDFOR
    \ENDIF
\end{algorithmic}
\end{algorithm}
\section{The MSign Optimizer: Breaking the Feedback Loop}

Based on our theoretical analysis and empirical observations, we propose a new optimizer to prevent stable rank collapse: the MSign optimizer. This method directly addresses the root cause by periodically restoring the stable rank of weight matrices via the matrix sign operation.

\subsection{The Matrix Sign Operation}

\begin{definition}[Matrix Sign Operator]
For any matrix $\mW \in \R^{m \times n}$ with \emph{reduced} (thin) SVD $\mW = \mU \mS \mV^T$, where $r = \text{rank}(\mW)$, $\mU \in \R^{m \times r}$, $\mS \in \R^{r \times r}$, and $\mV \in \R^{n \times r}$, we define:
\begin{equation}
\signm(\mW) = \mU \mV^T \in \R^{m \times n}.
\end{equation}
This operation sets all non-zero singular values to 1, creating a partial isometry that maximizes stable rank for a given matrix rank. Note that using the reduced SVD ensures the product $\mU \mV^T$ has the correct shape $m \times n$.
\end{definition}

\begin{figure*}[htp]
    \centering
    \includegraphics[width=0.24\linewidth]{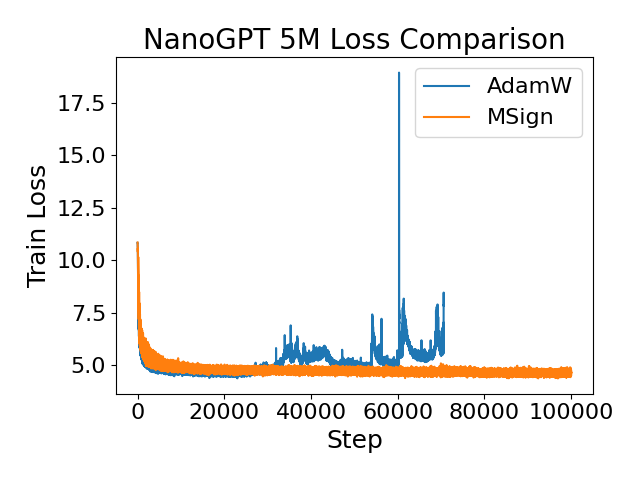}
    \includegraphics[width=0.24\linewidth]{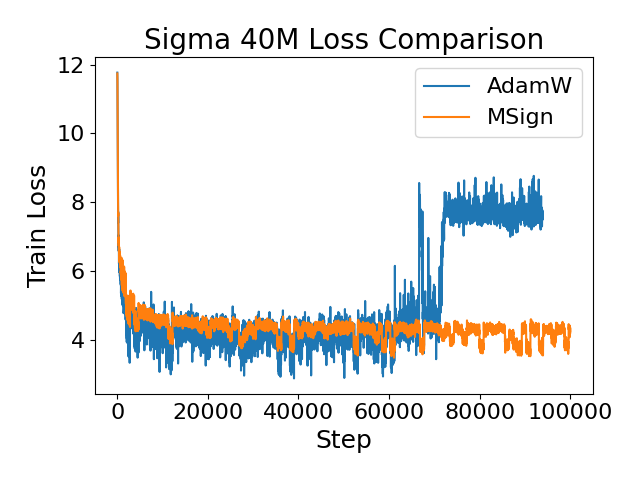}
    \includegraphics[width=0.24\linewidth]{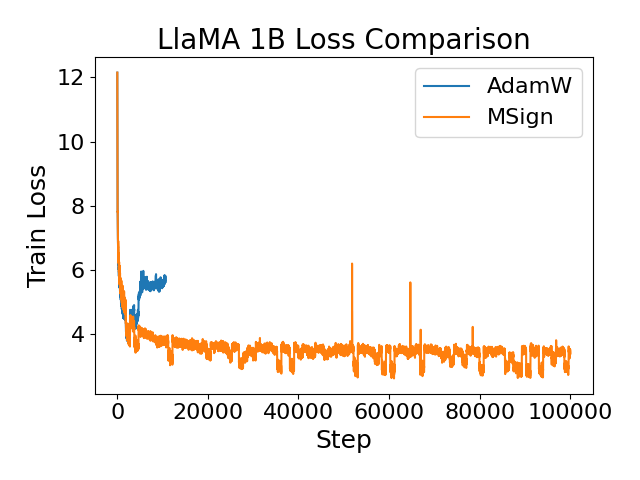}
    \includegraphics[width=0.24\linewidth]{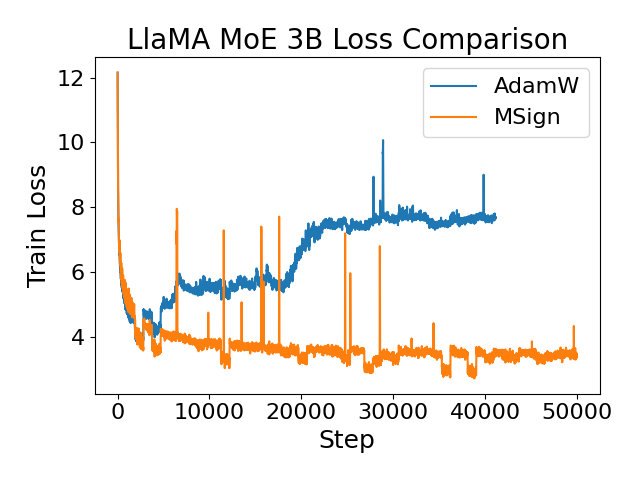}
    \includegraphics[width=0.24\linewidth]{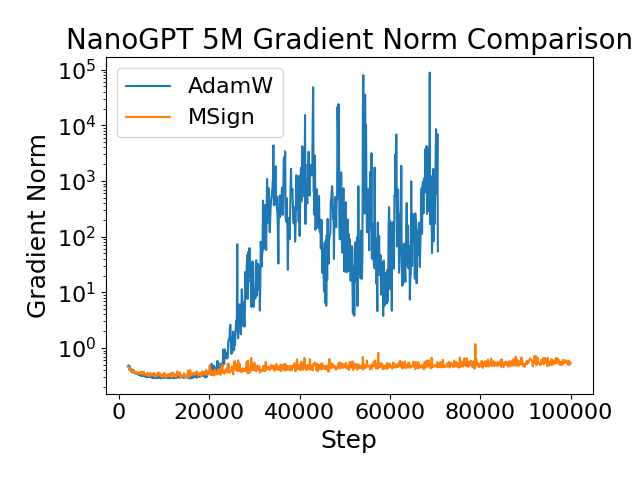}
    \includegraphics[width=0.24\linewidth]{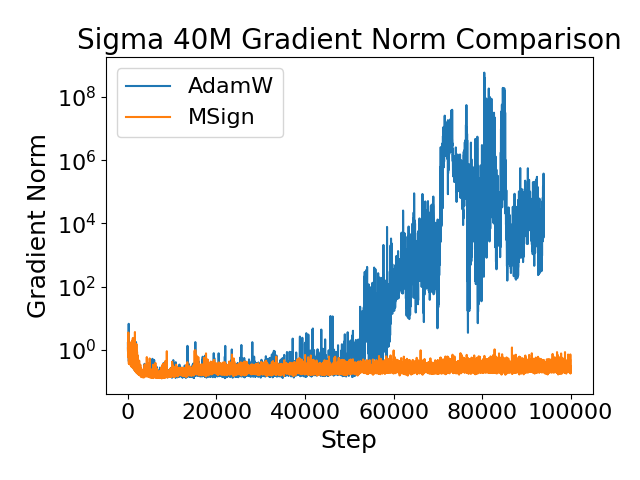}
    \includegraphics[width=0.24\linewidth]{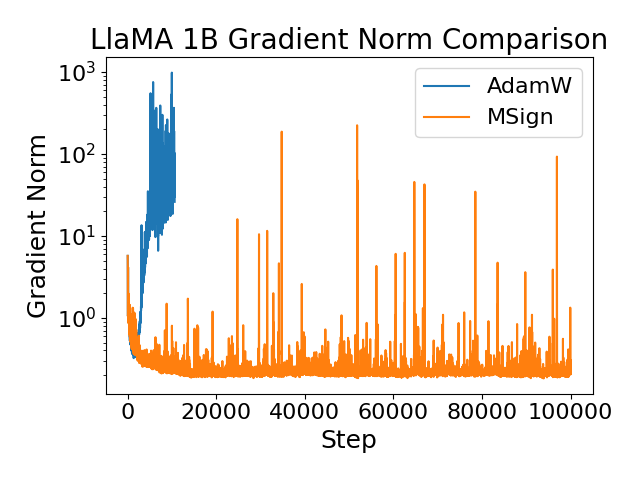}
    \includegraphics[width=0.24\linewidth]{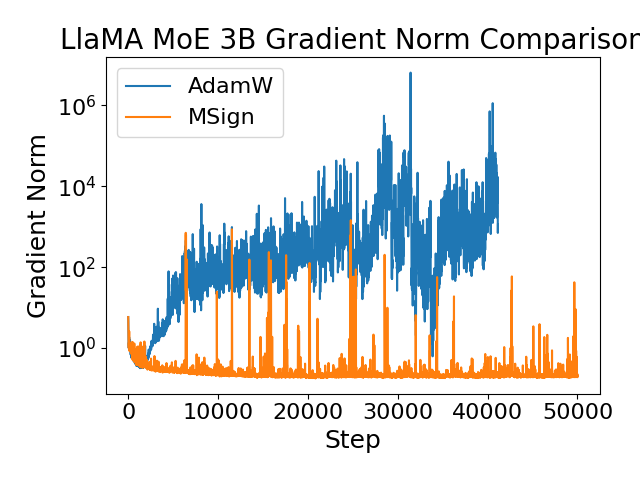}
    \caption{MSign prevents training failures across model scales. \textbf{Top row:} Training loss comparison between baseline (blue) and MSign (orange). Baseline training collapses with sudden loss spikes, while MSign maintains stable convergence. \textbf{Bottom row:} Corresponding gradient norm dynamics. Baseline runs exhibit exponential gradient explosion preceding collapse, while MSign keeps gradient norms bounded throughout training. Training is terminated after failure to conserve computational resources. Results demonstrate that MSign effectively breaks the stable rank collapse feedback loop identified in our theoretical analysis. From left to right: NanoGPT-5M, Sigma-40M, LLaMA-1B, LLaMA-MoE-3B.}
    \label{fig:msign-effect}
    \vspace{-10pt}
\end{figure*}
\subsection{Practical Implementation}
\paragraph{Scaling.} The matrix sign operation alters the scale of the weight matrix, necessitating rescaling. A straightforward approach is to preserve the original Frobenius norm:
\begin{equation}
\mW_{\text{new}} = \frac{\|\mW\|_F}{\|\signm(\mW)\|_F} \signm(\mW).
\end{equation}
In our current implementation, we adopt this Frobenius-norm preserving rescaling for MSign; however, when the stable rank is sufficiently low, this choice may excessively amplify minor singular values, and designing more principled rescaling schemes is left for future work.

\paragraph{Approximation Strategies.} Computing full SVD at every step is prohibitively expensive. We employ several practical strategies to reduce computational cost:

\paragraph{Periodic Application.} 
MSign is applied every $P$ steps (e.g., $P = 100$) rather than at every update. Our experiments demonstrate that this periodic application maintains effectiveness while substantially reducing computational overhead.

\paragraph{Selective Layer Targeting.}
Based on our empirical findings, MSign can be selectively applied to the most critical layers:
\begin{itemize}
\item \textbf{Attention-only}: Apply only to self-attention weights ($\mW_Q, \mW_K, \mW_V, \mW_O$)
\item \textbf{2D-parameters-only}: Apply to all 2D parameter tensors (excluding biases and layer norm parameters)
\end{itemize}
\subsection{Computational Cost Analysis}
\label{sec:msign}
The computational overhead of MSign depends on the frequency and scope of application. Table~\ref{tab:compute-cost} in Appendix~\ref{sec:compute-cost-details} provides a detailed comparison of computational costs for a typical transformer layer.

We define the overhead ratio $R$ as the ratio of additional FLOPs introduced by the MSign operation to the original FLOPs of a standard training step:
\begin{equation}
R = \frac{\text{FLOPs}_{\text{MSign}}}{\text{FLOPs}_{\text{original}} \times P},
\end{equation}
where $P$ is the application period (i.e., MSign is applied every $P$ steps).

From the detailed breakdown in Table~\ref{tab:compute-cost}, we obtain the concrete formula for the application:
\begin{equation}
\label{eq:overhead-ratio}
R =\frac{90d^3}{(72BTd^2+12BT^2d+O(BT^2+BTd)) \times P}.
\end{equation}
For a concrete example, consider a typical configuration with batch size $B=16$, sequence length $T=1024$, hidden dimension $d=2048$, and periodic application every $P=100$ steps. The MSign overhead per application for attention weights is $52d^3 \approx 4.47 \times 10^{11}$ FLOPs. The original per-step FLOPs are approximately $72BTd^2+12BT^2d \approx 5.36 \times 10^{12}$ FLOPs. Thus:
\begin{equation}
R \approx \frac{4.47 \times 10^{11}}{5.36 \times 10^{12} \times 100} \approx 0.08\%.
\end{equation}
\section{Experiment}
We validate our approach across multiple experimental settings of increasing scale. Our experiments demonstrate: (1) the effectiveness of MSign in preventing training failures, and (2) the minimal impact on training throughput.

\subsection{Experimental Setting}
We conduct experiments on four model configurations spanning different scales and architectures: NanoGPT-5M (No RoPE)~\footnote{\href{https://github.com/karpathy/nanoGPT}{https://github.com/karpathy/nanoGPT}}; Sigma-40M (Hybrid Attention)~\cite{qu2025sigma,hu2025sigmatiny}; LLaMA-1B (Full RoPE)~\cite{kumar2025zclip}; LLaMA-MOE-3B (Mixture of Experts) (Modified based on LLaMA-1B). Detailed configurations are provided in Table~\ref{tab:exp-settings} in Appendix~\ref{sec:exp-details}.




\begin{table*}[htp]
\centering
\caption{Training throughput (tokens/second) across model scales. Measured overhead significantly exceeds theoretical predictions due to implementation factors discussed in Section~\ref{sec:throughput-analysis}.}
\label{tab:throughput}
\begin{tabular}{lcccc}
\toprule
 & \textbf{NanoGPT-5M} & \textbf{Sigma-40M} & \textbf{LLaMA-1B} & \textbf{LLaMA-MoE-3B} \\
\midrule
AdamW & 102,708 & 6,504 & 1,742 & 544 \\
MSign & 105,199 & 6,097 & 1,640 & 520 \\
\midrule
Measured Overhead & $-2.4\%$ & $6.7\%$ & $6.2\%$ & $4.6\%$ \\
Theoretical Overhead (Eq.~\eqref{eq:overhead-ratio}) & $<0.01\%$ & $<0.01\%$ & $0.03\%$ & $0.09\%$ \\
\bottomrule
\end{tabular}
\vspace{-5pt}
\end{table*}
\subsection{Main Result}
\paragraph{MSign Prevents Training Failures.}
Figure~\ref{fig:msign-effect} demonstrates that MSign effectively prevents training collapse across all experimental settings. We analyze the results from multiple perspectives.

\textbf{Training Loss Dynamics (Top Row).} The top row of Figure~\ref{fig:msign-effect} shows training loss trajectories. For all four model scales, baseline training (blue curves) exhibits characteristic instability patterns: initial smooth convergence followed by sudden loss spikes and divergence. The collapse timing varies by model scale, NanoGPT-5M fails around step 30k, Sigma-40M around step 50k, LLaMA-1B around step 2k, and LLaMA-MoE-3B around step 3k, but the pattern is consistent across both dense and sparse architectures. In contrast, MSign-enabled training (orange curves) maintains stable convergence throughout, achieving comparable or better final loss values. Notably, MSign is equally effective for the MoE architecture, where the distributed nature of expert computation does not affect the attention-based instability mechanism.

\textbf{Gradient Norm Analysis (Bottom Row).} The bottom row reveals the underlying mechanism. Baseline runs show exponential gradient explosion (reaching $10^1$--$10^7$) immediately preceding each loss spike, confirming that gradient instability is the proximate cause of training failure. With MSign, gradient norms remain bounded within $10^{0}$ throughout training. The periodic structure visible in the MSign curves corresponds to the application period $P=100$, where each MSign application slightly perturbs the optimization trajectory before quickly stabilizing.


\subsection{Throughput Analysis}
\label{sec:throughput-analysis}
Table~\ref{tab:throughput} compares training throughput across model scales. For NanoGPT-5M, the measured overhead ($-2.4\%$) falls within system noise, consistent with the theoretical prediction ($<0.1\%$). For larger models, however, measured overhead (4.6--6.7\%) significantly exceeds theoretical values. This discrepancy arises from implementation factors not captured in FLOPs analysis: (1) \texttt{all\_gather} communication for distributed SVD, (2) disruption of FlashAttention kernel fusion and continuous stream execution, and (3) pipeline bubbles in distributed training. Despite these overheads, the 4--7\% throughput reduction remains modest compared to the computational waste from training failures.



\subsection{Ablation Study}

We conduct detailed ablation studies on the NanoGPT-5M and Sigma-40M configurations to identify the key factors affecting MSign effectiveness.

\paragraph{Ablation 1: Layer Selection: Attention vs. MLP.}
A critical finding is that \textbf{MSign must be applied to attention layers} to prevent training failures. Table~\ref{tab:layer-ablation} shows the results:

\begin{table}[h]
\centering
\caption{Layer selection ablation: test perplexity ($\downarrow$) on NanoGPT-5M and Sigma-40M.}
\label{tab:layer-ablation}
\small
\begin{tabular}{lcc}
\toprule
\textbf{Target Layers} & \textbf{NanoGPT-5M} & \textbf{Sigma-40M} \\
\midrule
AdamW & Failed & Failed \\
\midrule
MSign (All 2D)  & 102.6 & 74.00 \\
MSign (Attention only) & 118.6 & 75.68
 \\
MSign (MLP only)  & Failed & Failed \\
\bottomrule
\end{tabular}
\end{table}

Table~\ref{tab:layer-ablation} shows that applying MSign to MLP layers alone fails to prevent training collapse in both NanoGPT-5M and Sigma-40M, whereas attention-only application successfully stabilizes training. This aligns with our theoretical analysis: attention layers create the inter-layer Jacobian structure that propagates through the network. Applying MSign to all 2D parameters achieves the best perplexity on both models (e.g., 102.6 vs.\ 118.6 on NanoGPT-5M), indicating that including MLP layers significantly improves final model quality.

\paragraph{Ablation 2: Application Period $P$.}
The application period $P$ controls the trade-off between computational overhead and stable rank maintenance. Table~\ref{tab:period-ablation} shows results for different values of $P$ on NanoGPT-5M.

\begin{table}[h]
\centering
\caption{Application period ablation: test perplexity ($\downarrow$) and throughput (tokens/s) across models.}
\label{tab:period-ablation}
\small
\begin{tabular}{lccc}
\toprule
 \textbf{Period $P$}& \multicolumn{2}{c}{\textbf{Test PPL}} & \textbf{Throughput} \\
\cmidrule(lr){2-3} \cmidrule(lr){4-4}
 & \textbf{NanoGPT} & \textbf{Sigma} & \textbf{Sigma} \\
\midrule
$P=10$ & 103.9 & 92.7 & 18,236 \\
$P=100$ & 102.6 & 74.0 & 24,559 \\
$P=1000$ & 99.4 & 75.7 & 25,082 \\
$P=10000$ & 104.2 & \textbf{69.5} & 25,270 \\
$P=100000$ & Failed & Failed & --- \\
\bottomrule
\end{tabular}
\vspace{-10pt}
\end{table}

Table~\ref{tab:period-ablation} demonstrates the robustness of MSign across a wide range of application periods. All tested values of $P$ from 10 to 10,000 successfully prevent training collapse on both models in terms of final perplexity. However, as shown in Figure~\ref{fig:period-ablation} in the Appendix, for NanoGPT-5M with $P=10000$, the training dynamics exhibit noticeable instability: both loss and gradient norm show increased variance and occasional spikes compared to smaller $P$ values, indicating that the stable rank may temporarily drop below safe thresholds between MSign applications. Since $P=100$ already achieves acceptable computational overhead (Section~\ref{sec:throughput-analysis}) while maintaining stable training dynamics, we recommend $P=100$ as the conservative default choice rather than $P=1000$, prioritizing training stability over marginal perplexity improvements.
\section{Conclusion}
We identify and analyze the stable rank collapse feedback loop as a fundamental mechanism underlying LLM training instability. Low weight stable rank amplifies layer Jacobian norms, which combined with inter-layer alignment causes exponential gradient growth. The MSign optimizer breaks this feedback loop by periodically restoring stable rank via the matrix sign operation, effectively preventing training failures with minimal overhead ($<7\%$). Future directions include adaptive scheduling, fused kernels for reduced latency, and extensions to other training pathologies.

While our theoretical results provide a rigorous foundation for understanding the positive feedback mechanism of stable rank collapse (see Theorem~4.12), they rely on strong assumptions—particularly, the uniform negative correlation of input and output gradient projections. These structural conditions  may not universally hold in practice. Explicitly characterizing the full range of scenarios where the feedback loop provably dominates remains an open problem. We leave a complete characterization of these conditions and their relaxation to future work, which will further clarify the generality and boundaries of our theory.


\section*{Impact Statement}
This paper presents work whose goal is to advance the field of machine learning. There are many potential societal consequences of our work, none of which we feel must be specifically highlighted here.


\begin{thebibliography}{32}
\providecommand{\natexlab}[1]{#1}
\providecommand{\url}[1]{\texttt{#1}}
\expandafter\ifx\csname urlstyle\endcsname\relax
  \providecommand{\doi}[1]{doi: #1}\else
  \providecommand{\doi}{doi: \begingroup \urlstyle{rm}\Url}\fi

\bibitem[Arora et~al.(2019)Arora, Cohen, Hu, and Luo]{arora2019exact}
Arora, S., Cohen, N., Hu, W., and Luo, Y.
\newblock Implicit regularization in deep matrix factorization.
\newblock volume~32, 2019.

\bibitem[Ba et~al.(2016)Ba, Kiros, and Hinton]{ba2016layer}
Ba, L.~J., Kiros, J.~R., and Hinton, G.~E.
\newblock Layer normalization.
\newblock \emph{CoRR}, abs/1607.06450, 2016.
\newblock URL \url{http://arxiv.org/abs/1607.06450}.

\bibitem[Chowdhery et~al.(2023)Chowdhery, Narang, Devlin, Bosma, Mishra, Roberts, Barham, Chung, Sutton, Gehrmann, Schuh, Shi, Tsvyashchenko, Maynez, Rao, Barnes, Tay, Shazeer, Prabhakaran, Reif, Du, Hutchinson, Pope, Bradbury, Austin, Isard, Gur-Ari, Yin, Duke, Levskaya, Ghemawat, Dev, Michalewski, Garcia, Misra, Robinson, Fedus, Zhou, Ippolito, Luan, Lim, Zoph, Spiridonov, Sepassi, Dohan, Agrawal, Omernick, Dai, Pillai, Pellat, Lewkowycz, Moreira, Child, Polozov, Lee, Zhou, Wang, Saeta, Diaz, Firat, Catasta, Wei, Meier-Hellstern, Eck, Dean, Petrov, and Fiedel]{chowdhery2022palm}
Chowdhery, A., Narang, S., Devlin, J., Bosma, M., Mishra, G., Roberts, A., Barham, P., Chung, H.~W., Sutton, C., Gehrmann, S., Schuh, P., Shi, K., Tsvyashchenko, S., Maynez, J., Rao, A., Barnes, P., Tay, Y., Shazeer, N., Prabhakaran, V., Reif, E., Du, N., Hutchinson, B., Pope, R., Bradbury, J., Austin, J., Isard, M., Gur-Ari, G., Yin, P., Duke, T., Levskaya, A., Ghemawat, S., Dev, S., Michalewski, H., Garcia, X., Misra, V., Robinson, K., Fedus, L., Zhou, D., Ippolito, D., Luan, D., Lim, H., Zoph, B., Spiridonov, A., Sepassi, R., Dohan, D., Agrawal, S., Omernick, M., Dai, A.~M., Pillai, T.~S., Pellat, M., Lewkowycz, A., Moreira, E., Child, R., Polozov, O., Lee, K., Zhou, Z., Wang, X., Saeta, B., Diaz, M., Firat, O., Catasta, M., Wei, J., Meier-Hellstern, K., Eck, D., Dean, J., Petrov, S., and Fiedel, N.
\newblock {PaLM}: Scaling language modeling with {Pathways}.
\newblock \emph{Journal of Machine Learning Research}, 24\penalty0 (240):\penalty0 1--113, 2023.
\newblock ISSN 1532-4435.

\bibitem[Denil et~al.(2013)Denil, Shakibi, Dinh, Ranzato, and De~Freitas]{denil2013predicting}
Denil, M., Shakibi, B., Dinh, L., Ranzato, M., and De~Freitas, N.
\newblock Predicting parameters in deep learning.
\newblock volume~26, 2013.

\bibitem[Dong et~al.(2021)Dong, Cordonnier, and Loukas]{dong2021attention}
Dong, Y., Cordonnier, J.-B., and Loukas, A.
\newblock Attention is not all you need: Pure attention loses rank doubly exponentially with depth.
\newblock In \emph{International conference on machine learning}, pp.\  2793--2803. PMLR, 2021.

\bibitem[Fort et~al.(2019)Fort, Hu, and Lakshminarayanan]{fort2019deep}
Fort, S., Hu, H., and Lakshminarayanan, B.
\newblock Deep ensembles: A loss landscape perspective.
\newblock 2019.

\bibitem[Glorot \& Bengio(2010)Glorot and Bengio]{glorot2010understanding}
Glorot, X. and Bengio, Y.
\newblock Understanding the difficulty of training deep feedforward neural networks.
\newblock In \emph{Proceedings of the thirteenth international conference on artificial intelligence and statistics}, pp.\  249--256. JMLR Workshop and Conference Proceedings, 2010.

\bibitem[Hu et~al.(2022)Hu, yelong shen, Wallis, Allen-Zhu, Li, Wang, Wang, and Chen]{hu2022lora}
Hu, E.~J., yelong shen, Wallis, P., Allen-Zhu, Z., Li, Y., Wang, S., Wang, L., and Chen, W.
\newblock Lo{RA}: Low-rank adaptation of large language models.
\newblock In \emph{International Conference on Learning Representations}, 2022.
\newblock URL \url{https://openreview.net/forum?id=nZeVKeeFYf9}.

\bibitem[Hu et~al.(2025)Hu, Lin, Yang, Ding, Liu, Jiang, Wang, Chen, Guo, Xiong, Gao, Qu, Su, Cheng, and Gong]{hu2025sigmatiny}
Hu, Q., Lin, Z., Yang, Z., Ding, Y., Liu, X., Jiang, Y., Wang, R., Chen, T., Guo, Z., Xiong, Y., Gao, R., Qu, L., Su, J., Cheng, P., and Gong, Y.
\newblock {Sigma-MoE-Tiny} technical report, 2025.
\newblock URL \url{https://arxiv.org/abs/2512.16248}.

\bibitem[Huh et~al.(2024)Huh, Cheung, Wang, and Isola]{huh2024platonic}
Huh, M., Cheung, B., Wang, T., and Isola, P.
\newblock Position: The platonic representation hypothesis.
\newblock In \emph{Forty-first International Conference on Machine Learning}, 2024.
\newblock URL \url{https://openreview.net/forum?id=BH8TYy0r6u}.

\bibitem[Kaplan et~al.(2020)Kaplan, McCandlish, Henighan, Brown, Chess, Child, Gray, Radford, Wu, and Amodei]{kaplan2020scaling}
Kaplan, J., McCandlish, S., Henighan, T., Brown, T.~B., Chess, B., Child, R., Gray, S., Radford, A., Wu, J., and Amodei, D.
\newblock Scaling laws for neural language models.
\newblock 2020.
\newblock URL \url{https://arxiv.org/abs/2001.08361}.

\bibitem[Kingma \& Ba(2015)Kingma and Ba]{kingma2014adam}
Kingma, D.~P. and Ba, J.
\newblock Adam: {A} method for stochastic optimization.
\newblock In Bengio, Y. and LeCun, Y. (eds.), \emph{3rd International Conference on Learning Representations, {ICLR} 2015, San Diego, CA, USA, May 7-9, 2015, Conference Track Proceedings}, 2015.
\newblock URL \url{http://arxiv.org/abs/1412.6980}.

\bibitem[Kumar et~al.(2025)Kumar, Owen, Chowdhury, and G{"{u}}ra]{kumar2025zclip}
Kumar, A., Owen, L., Chowdhury, N.~R., and G{"{u}}ra, F.
\newblock Zclip: Adaptive spike mitigation for {LLM} pre-training, 2025.
\newblock URL \url{https://doi.org/10.48550/arXiv.2504.02507}.

\bibitem[Li et~al.(2018)Li, Farkhoor, Liu, and Yosinski]{li2018measuring}
Li, C., Farkhoor, H., Liu, R., and Yosinski, J.
\newblock Measuring the intrinsic dimension of objective landscapes.
\newblock In \emph{ICLR (Poster)}, 2018.
\newblock URL \url{https://openreview.net/forum?id=ryup8-WCW}.

\bibitem[Moniri et~al.(2024)Moniri, Lee, Hassani, and Dobriban]{molybog2023theory}
Moniri, B., Lee, D., Hassani, H., and Dobriban, E.
\newblock A theory of non-linear feature learning with one gradient step in two-layer neural networks.
\newblock In \emph{Proceedings of the 41st International Conference on Machine Learning}, ICML'24. JMLR.org, 2024.

\bibitem[Neyshabur et~al.(2017)Neyshabur, Bhojanapalli, McAllester, and Srebro]{NeyshaburBMS17aa}
Neyshabur, B., Bhojanapalli, S., McAllester, D., and Srebro, N.
\newblock A pac-bayesian approach to spectrally-normalized margin bounds for neural networks.
\newblock \emph{CoRR}, abs/1707.09564, 2017.
\newblock URL \url{http://arxiv.org/abs/1707.09564}.

\bibitem[Pascanu et~al.(2012)Pascanu, Mikolov, and Bengio]{pascanu2013difficulty}
Pascanu, R., Mikolov, T., and Bengio, Y.
\newblock On the difficulty of training recurrent neural networks.
\newblock \emph{30th International Conference on Machine Learning, ICML 2013}, pp.\  1310--1318, 11 2012.

\bibitem[Pennington et~al.(2017)Pennington, Schoenholz, and Ganguli]{pennington2017resnet}
Pennington, J., Schoenholz, S.~S., and Ganguli, S.
\newblock Resurrecting the sigmoid in deep learning through dynamical isometry: theory and practice.
\newblock \emph{CoRR}, abs/1711.04735, 2017.
\newblock URL \url{http://arxiv.org/abs/1711.04735}.

\bibitem[Philipp et~al.(2017)Philipp, Song, and Carbonell]{philipp2018exploding}
Philipp, G., Song, D., and Carbonell, J.~G.
\newblock The exploding gradient problem demystified - definition, prevalence, impact, origin, tradeoffs, and solutions.
\newblock \emph{arXiv preprint arXiv:1712.05577}, 2017.

\bibitem[Qu et~al.(2025)Qu, Ren, Cheng, Gao, Wang, Chen, Liu, Zhang, Gong, Xiong, Ding, Jiang, Lin, Guo, and Yang]{qu2025sigma}
Qu, L., Ren, L., Cheng, P., Gao, R., Wang, R., Chen, T., Liu, X., Zhang, X., Gong, Y., Xiong, Y., Ding, Y., Jiang, Y., Lin, Z., Guo, Z., and Yang, Z.
\newblock {SIGMA}: An {AI}-empowered training stack on early-life hardware, 2025.
\newblock URL \url{https://arxiv.org/abs/2512.13488}.

\bibitem[Rudelson \& Vershynin(2007)Rudelson and Vershynin]{rudelson2007sampling}
Rudelson, M. and Vershynin, R.
\newblock Sampling from large matrices: An approach through geometric functional analysis.
\newblock \emph{J. ACM}, 54\penalty0 (4):\penalty0 21–es, July 2007.
\newblock ISSN 0004-5411.
\newblock \doi{10.1145/1255443.1255449}.
\newblock URL \url{https://doi.org/10.1145/1255443.1255449}.

\bibitem[Sanyal et~al.(2020)Sanyal, Torr, and Dokania]{sanyal2020stable}
Sanyal, A., Torr, P.~H., and Dokania, P.~K.
\newblock Stable rank normalization for improved generalization in neural networks and gans.
\newblock In \emph{International Conference on Learning Representations}, 2020.
\newblock URL \url{https://openreview.net/forum?id=H1enKkrFDB}.

\bibitem[Saxe et~al.(2013)Saxe, McClelland, and Ganguli]{saxe2014exact}
Saxe, A.~M., McClelland, J.~L., and Ganguli, S.
\newblock Exact solutions to the nonlinear dynamics of learning in deep linear neural networks.
\newblock volume abs/1312.6120, 2013.
\newblock URL \url{https://api.semanticscholar.org/CorpusID:17272965}.

\bibitem[Trefethen \& Bau~III(1997)Trefethen and Bau~III]{trefethen1997numerical}
Trefethen, L.~N. and Bau~III, D.
\newblock \emph{Numerical Linear Algebra}.
\newblock SIAM, Philadelphia, 1997.

\bibitem[Vershynin(2018)]{vershynin2018high}
Vershynin, R.
\newblock \emph{High-Dimensional Probability: An Introduction with Applications in Data Science}, volume~47.
\newblock Cambridge university press, 2018.

\bibitem[Wortsman et~al.(2024)Wortsman, Liu, Xiao, Everett, Alemi, Adlam, Co-Reyes, Gur, Kumar, Novak, Pennington, Sohl-dickstein, Xu, Lee, Gilmer, and Kornblith]{wortsman2023small}
Wortsman, M., Liu, P.~J., Xiao, L., Everett, K., Alemi, A., Adlam, B., Co-Reyes, J.~D., Gur, I., Kumar, A., Novak, R., Pennington, J., Sohl-dickstein, J., Xu, K., Lee, J., Gilmer, J., and Kornblith, S.
\newblock Small-scale proxies for large-scale transformer training instabilities.
\newblock In \emph{International Conference on Learning Representations}, 2024.

\bibitem[Xiao et~al.(2018)Xiao, Bahri, Sohl-Dickstein, Schoenholz, and Pennington]{xiao2018dynamical}
Xiao, L., Bahri, Y., Sohl-Dickstein, J., Schoenholz, S., and Pennington, J.
\newblock Dynamical isometry and a mean field theory of {CNNs}: How to train 10,000-layer vanilla convolutional neural networks.
\newblock In \emph{International conference on machine learning}, pp.\  5393--5402. PMLR, 2018.

\bibitem[Yang \& Schoenholz(2017)Yang and Schoenholz]{yang2017mean}
Yang, G. and Schoenholz, S.~S.
\newblock Mean field residual networks: On the edge of chaos.
\newblock \emph{CoRR}, abs/1712.08969, 2017.
\newblock URL \url{http://arxiv.org/abs/1712.08969}.

\bibitem[Yang et~al.(2022)Yang, Hu, Babuschkin, Sidor, Liu, Farhi, Ryder, Pachocki, Chen, and Gao]{yang2022tensor}
Yang, G., Hu, E.~J., Babuschkin, I., Sidor, S., Liu, X., Farhi, D., Ryder, N., Pachocki, J., Chen, W., and Gao, J.
\newblock Tensor programs {V}: Tuning large neural networks via zero-shot hyperparameter transfer.
\newblock \emph{CoRR}, abs/2203.03466, 2022.
\newblock \doi{10.48550/ARXIV.2203.03466}.
\newblock URL \url{https://doi.org/10.48550/arXiv.2203.03466}.

\bibitem[Zeng et~al.(2023)Zeng, Liu, Du, Wang, Lai, Ding, Yang, Xu, Zheng, Xia, Tam, Ma, Xue, Zhai, Chen, Zhang, Dong, and Tang]{zeng2022glm}
Zeng, A., Liu, X., Du, Z., Wang, Z., Lai, H., Ding, M., Yang, Z., Xu, Y., Zheng, W., Xia, X., Tam, W.~L., Ma, Z., Xue, Y., Zhai, J., Chen, W., Zhang, P., Dong, Y., and Tang, J.
\newblock {GLM-130B}: An open bilingual pre-trained model.
\newblock 2023.
\newblock URL \url{https://arxiv.org/abs/2210.02414}.

\bibitem[Zhang et~al.(2022)Zhang, Roller, Goyal, Artetxe, Chen, Chen, Dewan, Diab, Li, Lin, et~al.]{zhang2022opt}
Zhang, S., Roller, S., Goyal, N., Artetxe, M., Chen, M., Chen, S., Dewan, C., Diab, M., Li, X., Lin, X.~V., et~al.
\newblock {OPT}: Open pre-trained transformer language models.
\newblock \emph{arXiv preprint arXiv:2205.01068}, 2022.

\bibitem[Zhao et~al.(2024)Zhao, Zhang, Chen, Wang, Anandkumar, and Tian]{galore2023}
Zhao, J., Zhang, Z., Chen, B., Wang, Z., Anandkumar, A., and Tian, Y.
\newblock {GaLore}: Memory-efficient {LLM} training by gradient low-rank projection.
\newblock 2024.

\end{thebibliography}


\newpage
\appendix
\onecolumn

\section{Experimental Details}
\label{sec:exp-details}

\subsection{Model Configurations}

Table~\ref{tab:exp-settings} summarizes the experimental configurations across four model scales used in our experiments.

\begin{table*}[h]
\centering
\caption{Experimental configurations across four model scales}
\label{tab:exp-settings}
\resizebox{\textwidth}{!}{
\begin{tabular}{lcccc}
\toprule
\textbf{Setting} & \textbf{NanoGPT-5M} & \textbf{Sigma-40M} & \textbf{LLaMA-1B} & \textbf{LLaMA-MoE-3B} \\
\midrule
Parameters & 5M & 40M & 1B & 1B (active) / 3B (total) \\
Layers & 24 & 28 & 16 & 16 \\
Hidden dim & 48 & 128 & 2048 & 2048 \\
Attention heads & 6 & 4 & 16 & 16 \\
Attention type & MHA & MHA/MLA alternating & MHA & MHA \\
Architecture & Dense & Dense & Dense & MoE (16 experts, top-4) \\
MLP intermediate & 192 & 448 & 5440 & 1360 per expert \\
Position encoding & None & RoPE (MHA only) & RoPE & RoPE \\
Layernorm & LayerNorm & RMSNorm & RMSNorm & RMSNorm \\
Activation & GELU & SiLU & SiLU & SiLU \\
Initial Std$^{1,2}$ & 0.02 & 0.006 & 0.02 & 0.02 \\
Dataset & OpenWebText & Nemotron-cc & Nemotron-cc & Nemotron-cc \\
Batch size (tokens) & 123k & 66k & 66k & 66k \\
Sequence length & 256 & 2048 & 2048 & 2048 \\
Learning rate$^{3}$ & $6 \times 10^{-4}$ & $3.5 \times 10^{-4}$ & $1 \times 10^{-3}$ & $1 \times 10^{-3}$ \\
Optimizer & AdamW & AdamW & AdamW & AdamW \\
\bottomrule
\end{tabular}
}
\vspace{0.5em}
\small{$^1$For NanoGPT-5M and Sigma-40M, the $\mu$P scaling rules are applied: the initial 2D weights in transformer layers are scaled by $4\times$, and query weights ($\mW_Q$) are initialized to zero.\\
$^2$For NanoGPT-5M and Sigma-40M, the output projection weights ($\mW_O,\mW_{down}$) are further divided by $\sqrt{2\cdot\text{\#layers}}$ as part of the original architecture design.\\
$^{3}$For NanoGPT-5M and Sigma-40M, the $\mu$P scaling rules are applied: the learning rate for 2D weights in transformer layers are scaled by $16\times$, and for embedding weights are scaled by $4\times$.
}
\end{table*}

\subsection{Common Training Settings}

The following hyperparameters and architectural choices are shared across all experiments unless otherwise specified:

\paragraph{Optimizer settings.} We use the AdamW optimizer with $\beta_1 = 0.9$, $\beta_2 = 0.95$, and $\epsilon = 10^{-8}$. Weight decay is set to $0.1$ for all experiments. Gradient clipping with max norm $1.0$ is applied in all experiments. For the learning rate scheduler, we warm up for 2000 steps and linearly decay to $1/10$ learning rate.

\paragraph{Architecture details.} Bias terms are disabled in all linear layers (attention projections and MLP layers). Pre-LayerNorm is used in all transformer blocks.

\subsection{Computational Cost Details}
\label{sec:compute-cost-details}

Table~\ref{tab:compute-cost} provides a detailed breakdown of computational costs for applying MSign to a typical transformer layer. The SVD computation dominates the MSign cost, with complexity $O(d^3)$ per weight matrix. However, when amortized over $P$ steps (typically $P=100$), the overhead becomes negligible compared to the forward and backward pass costs that scale as $O(BTd^2)$.

\begin{table*}[h]
\centering
\caption{Computational cost comparison per layer (hidden dim $d$, intermediate dim $4d$, batch size $B$, sequence length $T$, applied every $P$ steps)}
\label{tab:compute-cost}
\begin{tabular}{llcc}
\toprule
\textbf{Component} & \textbf{FLOPs per step} & \textbf{MSign FLOPs}$^\dagger$ & \textbf{Amortized MSign} \\
\midrule
\midrule
\multicolumn{4}{l}{\textbf{\textit{Attention Block }}} \\
\midrule
$\mW_Q/ \mW_K/ \mW_V/\mW_O$ & $6  BTd^2$ & $4 \times 13d^3$ & $52d^3/P$ \\
Attention computation & $ 12BT^2d+O(BT^2)$ & 0 & 0 \\
Total & $24BTd^2 + 12BT^2d+O(BT^2)$ & $52d^3$ & $52d^3/P$ \\
\midrule
\midrule
\multicolumn{4}{l}{\textit{\textbf{MLP Block}}} \\
\midrule
$\mW_{up}/\mW_{down}$ ($d \to 4d/4d \to d$) & $24BTd^2$ & $2 \times 19d^3$ & $38d^3/P$ \\
Activation & $O(BTd)$ & 0 & 0 \\
Total & $48BTd^2+O(BTd)$ & $38d^3$ & $38d^3/P$ \\
\midrule
\midrule
\multicolumn{4}{l}{\textit{\textbf{Entire Layer }}} \\
\midrule
Original & $72BTd^2+12BT^2d+O(BT^2+BTd)$ & 0 & 0 \\
+ MSign (Attn only) & $72BTd^2+12BT^2d+O(BT^2+BTd)$ & $52d^3$ & $52d^3/P$ \\
+ MSign (All 2D) & $72BTd^2+12BT^2d+O(BT^2+BTd)$ & $90d^3$ & $90d^3/P$ \\
\bottomrule
\end{tabular}
\vspace{0.5em}

\small{$^\dagger$SVD FLOPs computed as $2mn\min(m,n) + 11\min(m,n)^3$ per matrix~\citep{trefethen1997numerical}. For $d \times d$ matrix: $13d^3$; for $d \times 4d$ matrix: $19d^3$.}
\end{table*}
\subsection{Throughput Model Analysis}
\label{sec:throughput-model}

To quantify the relationship between application period $P$ and training throughput, we fit a simple analytical model. Let $f$ denote the baseline computation per token (fixed), and $F$ denote the additional computation per MSign application (fixed). When the period is $P$, the amortized overhead per token is $F/P$. The throughput $T(P)$ (tokens/s) can be modeled as:
\begin{equation}
T(P) = \frac{T_\infty}{1 + r/P}, \quad \text{where } r = F/f,
\end{equation}
and $T_\infty$ is the asymptotic throughput as $P \to \infty$ (i.e., baseline without MSign overhead).

Using the Sigma-40M throughput measurements from Table~\ref{tab:period-ablation}, we perform least-squares fitting by linearizing: $1/T(P) = (1/T_\infty)(1 + r/P)$. This yields:
\begin{equation}
T_\infty \approx 25,\!350 \text{ tokens/s}, \quad r \approx 3.9.
\end{equation}

The fitted model predicts:
\begin{center}
\begin{tabular}{lcccc}
\toprule
$P$ & 10 & 100 & 1000 & 10000 \\
\midrule
Measured & 18,236 & 24,559 & 25,082 & 25,270 \\
Predicted & 18,273 & 24,399 & 25,251 & 25,340 \\
\bottomrule
\end{tabular}
\end{center}

The close agreement validates the model. However, the fitted $r \approx 3.9$ significantly exceeds the theoretical prediction. From Section~\ref{sec:msign}, the FLOPs-based overhead ratio is $R = 52d^3 / (72BTd^2 \cdot P) < 0.1\%$ for typical configurations, implying $r_{\text{theory}} \ll 1$.

This gap arises from the implementation factors discussed in Section~\ref{sec:throughput-analysis}: (1) \texttt{all\_gather} synchronization latency for distributed SVD computation, (2) disruption of FlashAttention kernel fusion and continuous CUDA stream execution, and (3) pipeline bubbles in distributed training. These factors introduce latency-dominated overhead that scales poorly with batch size, explaining why the effective $r$ far exceeds the FLOPs-based prediction. Future work on asynchronous MSign execution and fused SVD kernels could potentially close this gap.

\subsection{Application Period Analysis}
\label{sec:period-analysis}

Figure~\ref{fig:period-ablation} provides detailed training dynamics for different application periods $P$ on NanoGPT-5M. The left figure shows training loss trajectories, while the right figure shows gradient norm evolution.

\begin{figure*}[h]
\centering
\includegraphics[width=0.48\linewidth]{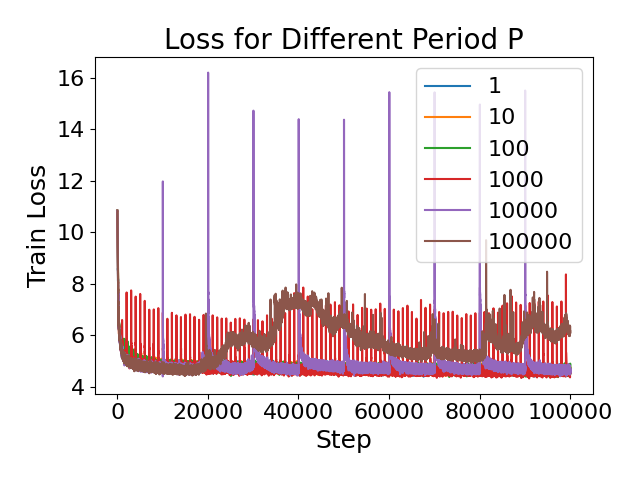}
\includegraphics[width=0.48\linewidth]{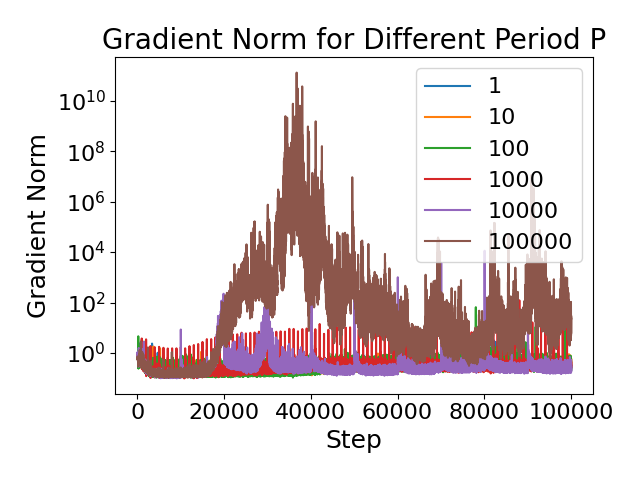}
\caption{Training dynamics under different MSign application periods on NanoGPT-5M. \textbf{Left:} Training loss comparison. \textbf{Right:} Gradient norm comparison. While all periods from $P=10$ to $P=10000$ eventually converge, $P=10000$ exhibits noticeably higher gradient norm in step 20000 to 40000, indicating intermittent instability when MSign applications are too infrequent.}
\label{fig:period-ablation}
\end{figure*}

Several observations emerge from this analysis:
\begin{itemize}
\item \textbf{$P=10$ and $P=100$}: Both show smooth, stable training dynamics with minimal variance in loss and gradient norm. The frequent MSign applications effectively maintain stable rank above critical thresholds throughout training.
\item \textbf{$P=1000$}: Training remains stable but shows slightly increased variance compared to smaller $P$ values. The longer intervals between MSign applications allow some transient stable rank decline, but recovery occurs before instability develops.
\item \textbf{$P=10000$}: While training eventually converges (PPL 104.2), the dynamics exhibit clear signs of intermittent instability, the gradient norm shows periodic spikes and the loss curve has higher variance. This suggests that 10,000 steps is near the boundary where stable rank can decline sufficiently to trigger partial feedback loop activation before the next MSign intervention.
\end{itemize}

These findings justify our recommendation of $P=100$ as the default: it provides a comfortable safety margin against instability while incurring negligible computational overhead.
\section{Proofs of Main Results}
\label{sec:proofs}

\subsection{Proof of Theorem~\ref{thm:jacobian-product} (Jacobian Product Norm Lower Bound)}
\label{proof:jacobian-product}

We first establish a key lemma for two-matrix products.
\begin{lemma}[Alignment-Preserving Product Bound]
    For matrices $\mA, \mB$ with SVD $\mA = \mU_A \mS_A \mV_A^T$ and $\mB = \mU_B \mS_B \mV_B^T$:
\begin{equation}
\|\mA\mB\|_2 \geq \|\mA\|_2 \|\mB\|_2 \cdot |\vv_{A,1}^T \vu_{B,1}|,
\end{equation}
where $\vv_{A,1}$ is the top right singular vector of $\mA$ and $\vu_{B,1}$ is the top left singular vector of $\mB$.
\end{lemma} 
\textit{Proof:} Let $\sigma_{A,1} = \|\mA\|_2$ and $\sigma_{B,1} = \|\mB\|_2$. We have:
\begin{equation}
\|\mA\mB\|_2 \geq \|\mA \mB \vv_{B,1}\|_2 = \|\mA \sigma_{B,1} \vu_{B,1}\|_2 = \sigma_{B,1} \|\mA \vu_{B,1}\|_2.
\end{equation}
Expanding $\vu_{B,1}$ in the basis of $\mA$'s right singular vectors:
\begin{equation}
\|\mA \vu_{B,1}\|_2 = \|\sum_i \sigma_{A,i} (\vv_{A,i}^T \vu_{B,1}) \vu_{A,i}\|_2 \geq \sigma_{A,1} |\vv_{A,1}^T \vu_{B,1}|,
\end{equation}
where the inequality uses $|\vv_{A,1}^T \vu_{B,1}| \leq 1$ and that $\{\vu_{A,i}\}$ is orthonormal.

Since $\text{Align}(\mA, \mB) = |\vv_{A,1}^T \vu_{B,1}|$ by definition, we obtain:
\begin{equation}
\|\mA\mB\|_2 \geq \|\mA\|_2 \|\mB\|_2 \cdot \text{Align}(\mA, \mB).
\end{equation}
Applying this recursively to the Jacobian product:
\begin{align}
\|\mJ^{(L)} \cdots \mJ^{(1)}\|_2 &\geq \|\mJ^{(L)}\|_2 \|\mJ^{(L-1)} \cdots \mJ^{(1)}\|_2 \cdot \text{Align}(L, L-1) \\
&\geq \|\mJ^{(L)}\|_2 \|\mJ^{(L-1)}\|_2 \|\mJ^{(L-2)} \cdots \mJ^{(1)}\|_2 \cdot a^2 \\
&\geq \cdots \geq M^L \cdot a^{L-1} = \frac{(aM)^L}{a}.
\end{align}
$\square$
\subsection{Proof of Theorem~\ref{thm:srank-jacobian} (Stable Rank Controls Jacobian Norm: Linear Layer)}
\label{proof:srank-jacobian}

Consider a linear layer that computes $\vh_{out} = \mW \vh_{in}$, where $\mW \in \R^{m \times n}$ is the weight matrix. The Jacobian of this transformation is:
\begin{equation}
\mJ = \frac{\partial \vh_{out}}{\partial \vh_{in}} = \mW,
\end{equation}
since the linear map $\vh \mapsto \mW\vh$ has constant derivative equal to $\mW$ itself.

The operator norm (also called spectral norm or 2-norm) of a matrix equals its largest singular value:
\begin{equation}
\|\mW\|_2 = \sigma_1(\mW) = \max_{\|\vx\|_2 = 1} \|\mW \vx\|_2.
\end{equation}
This represents the maximum factor by which the matrix can stretch any unit vector, i.e., the worst-case signal amplification.

From the definition of stable rank (Definition 1):
\begin{equation}
\srank(\mW) = \frac{\|\mW\|_F^2}{\|\mW\|_2^2} = \frac{\sum_i \sigma_i^2}{\sigma_1^2},
\end{equation}
where $\sigma_1 \geq \sigma_2 \geq \ldots \geq 0$ are the singular values of $\mW$.

Let $F = \|\mW\|_F$ denote the Frobenius norm, which is fixed by assumption. Substituting into the stable rank definition:
\begin{equation}
\srank(\mW) = \frac{F^2}{\|\mW\|_2^2}.
\end{equation}
Solving for the operator norm:
\begin{equation}
\|\mW\|_2^2 = \frac{F^2}{\srank(\mW)} \quad \Rightarrow \quad \|\mW\|_2 = \frac{F}{\sqrt{\srank(\mW)}}.
\end{equation}
This formula reveals a fundamental trade-off: for a matrix with fixed total ``energy'' (Frobenius norm $F$), the operator norm, which determines signal amplification, is inversely proportional to the square root of stable rank. 

Intuitively, stable rank measures how ``spread out'' the singular values are. When $\srank(\mW) \approx \text{rank}(\mW)$ (high stable rank), the singular values are roughly uniform, and $\|\mW\|_2 \approx F/\sqrt{\text{rank}(\mW)}$ is relatively small. When $\srank(\mW) \approx 1$ (low stable rank), almost all energy is concentrated in the top singular value, and $\|\mW\|_2 \approx F$, the maximum possible for given Frobenius norm. $\square$

\subsection{Proof of Theorem~\ref{thm:jacobian-attention} (Jacobian Norm Bound: Attention Layer)}
\label{proof:jacobian-attention}

We analyze the Jacobian of a single-head attention layer. The output is $\mY = \mA \mH \mW_V \mW_O$, where $\mA = \text{softmax}(\frac{\mH\mW_Q \mW_K^T \mH^T}{\sqrt{d_k}})$. The derivative of $\mY$ with respect to $\mH$ is a 4th-order tensor. To manage this, we use the Fréchet derivative, which is a linear operator $D_\mH(\Delta\mH)$ that approximates the change in $\mY$ for a small perturbation $\Delta\mH$.

Applying the product rule for Fréchet derivatives:
\begin{equation}
D_\mH(\Delta\mH) = [D_\mH\mA(\Delta\mH)](\mH\mW_V\mW_O) + \mA [D_\mH(\mH\mW_V\mW_O)(\Delta\mH)].
\end{equation}
The second term is straightforward: $D_\mH(\mH\mW_V\mW_O)(\Delta\mH) = (\Delta\mH)\mW_V\mW_O$.
Taking operator norms:
\begin{equation}
\|D\mY(\mH)[\Delta\mH]\|_2 \leq \|D\mA(\mH)[\Delta\mH]\|_2 \|\mH\mW_V\mW_O\|_2 + \|\mA\|_2 \|\Delta\mH \mW_V\mW_O\|_2.
\end{equation}
Dividing by $\|\Delta\mH\|_2$ and taking the supremum gives:
\begin{equation}
\|D\mY(\mH)\|_2 \leq \|D\mA(\mH)\|_2 \|\mH\|_2 \|\mW_V\|_2 \|\mW_O\|_2 + \|\mA\|_2 \|\mW_V\|_2 \|\mW_O\|_2.
\end{equation}

\textbf{Bounding the value gradient pathway.} For the second term, define $f(\mH) = \mH \mW_V \mW_O$. Since this is a linear function of $\mH$, its Fréchet derivative at any point $\mH$ is:
\begin{equation}
Df(\mH)[\Delta\mH] = \Delta\mH \mW_V \mW_O.
\end{equation}
The operator norm of this linear map is:
\begin{equation}
\|Df(\mH)\|_2 = \sup_{\|\Delta\mH\|_2=1} \|\Delta\mH \mW_V \mW_O\|_2 = \|\mW_V \mW_O\|_2 \leq \|\mW_V\|_2 \|\mW_O\|_2.
\end{equation}

\textbf{Bounding the attention gradient pathway.} The first term requires analyzing how the attention matrix $\mA$ changes with $\mH$. The attention matrix is computed as:
\begin{equation}
\mA = \text{softmax}\left(\frac{\mS}{\sqrt{d_k}}\right), \quad \text{where } \mS = \mQ \mK^T = \mH \mW_Q \mW_K^T \mH^T.
\end{equation}
Note that $\mS$ depends \emph{quadratically} on $\mH$, making this pathway more complex to analyze.

\textbf{Softmax Fréchet derivative.} Consider the softmax function $f(\vs) = \text{softmax}(\vs) = \frac{e^{\vs}}{\sum_j e^{s_j}}$ applied row-wise to $\mS$. For the $i$-th row, let $a_i = f(\vs_i)$ where $\vs_i$ is the $i$-th row of $\mS$. The Fréchet derivative $Df(\vs_i): \R^n \to \R^n$ is a linear map with matrix representation:
\begin{equation}
Df(\vs_i)[\Delta\vs_i] = (\text{diag}(a_i) - a_i a_i^T) \Delta\vs_i.
\end{equation}
This can be verified by computing: for each component, $\frac{\partial a_{i,k}}{\partial s_{i,j}} = a_{i,k}(\delta_{jk} - a_{i,j})$.

\textbf{Bounding the operator norm.} Let $\mJ_i = \text{diag}(a_i) - a_i a_i^T$. To bound $\|Df(\vs_i)\|_2 = \|\mJ_i\|_2$, we examine its quadratic form. For any $\vx \in \R^n$:
\begin{align}
\vx^T \mJ_i \vx &= \vx^T \text{diag}(a_i) \vx - \vx^T a_i a_i^T \vx \\
&= \sum_j a_{i,j} x_j^2 - \left(\sum_j a_{i,j} x_j\right)^2 \\
&= \mathbb{E}[X^2] - (\mathbb{E}[X])^2 = \text{Var}(X),
\end{align}
where we interpret $j$ as a random index sampled with probability $a_{i,j}$ and $X = x_j$ as the corresponding random variable. Since variance is non-negative, $\mJ_i \succeq 0$.

For a positive semi-definite matrix, the spectral norm equals the largest eigenvalue, which is bounded by the trace:
\begin{equation}
\|Df(\vs_i)\|_2 = \lambda_{\max}(\mJ_i) \leq \text{tr}(\mJ_i) = \sum_j a_{i,j} - \sum_j a_{i,j}^2 = 1 - \|a_i\|_2^2.
\end{equation}
Here we used $\sum_j a_{i,j} = 1$. Factoring gives:
\begin{equation}
\|Df(\vs_i)\|_2 \leq 1-\left\|a_i\right\|_2^2 = (1-\|a_i\|_2)(1+\|a_i\|_2) \leq 2(1-\max a_i),
\end{equation}
where the last inequality follows from $\|a_i\|_2 \geq \max a_i$ and $1 + \|a_i\|_2 \leq 2$.

\textbf{Connecting to logit margin.} The bound $1 - \max a_i$ measures how ``spread out'' the attention distribution is. When attention is sharply peaked, $\max a_i \approx 1$ and $\|Df(\vs_i)\|_2$ is small. We can express this in terms of the logit margin $\gamma_i = \max_j \mS_{i,j} - \text{second\_max}_j \mS_{i,j}$. By definition of softmax:
\begin{equation}
    \max a_i = \frac{e^{\max \vs_i}}{\sum_j e^{\vs_{i,j}}} \ge \frac{e^{\max \vs_i}}{e^{\max \vs_i} + \sum_{j \ne \arg\max_k \vs_{i,k}} e^{\vs_{i,j}}}.
\end{equation}
Since there are at most $n-1$ non-maximal terms, each at most $e^{\max \vs_i - \gamma_i}$:
\begin{equation}
    \max a_i \ge \frac{e^{\max \vs_i}}{e^{\max \vs_i} + (n-1)e^{\max \vs_i - \gamma_i}} = \frac{1}{1 + (n-1)e^{-\gamma_i}}.
\end{equation}
Rearranging and using $\frac{x}{1+x} \leq \min(x, 1)$ for $x \geq 0$:
\begin{equation}
1 - \max a_i \leq \frac{(n-1)e^{-\gamma_i}}{1 + (n-1)e^{-\gamma_i}} \leq \min((n-1)e^{-\gamma_i}, 1)\triangle\chi.
\end{equation}
Since the rows of $\mA$ are computed independently, the Fréchet derivative of the full softmax is block-diagonal and its operator norm is the maximum over rows:
\begin{equation}
    \left\|D(\text{softmax})(\mS)\right\|_2 = \max_i \|Df(\vs_i)\|_2 \leq 2 \chi.
\end{equation}

\textbf{Jacobian of attention logits.} For the attention logits $\mS(\mH) = \mH \mW_Q \mW_K^T \mH^T$, we compute the Fréchet derivative. Using the product rule for bilinear forms:
\begin{equation}
D\mS(\mH)[\Delta\mH] = \Delta\mH \mW_Q \mW_K^T \mH^T + \mH \mW_Q \mW_K^T (\Delta\mH)^T.
\end{equation}
The operator norm is bounded by:
\begin{align}
\|D\mS(\mH)\|_2 &= \sup_{\|\Delta\mH\|_2=1} \|\Delta\mH \mW_Q \mW_K^T \mH^T + \mH \mW_Q \mW_K^T (\Delta\mH)^T\|_2 \\
&\leq \|\mW_Q \mW_K^T \mH^T\|_2 + \|\mH \mW_Q \mW_K^T\|_2 \\
&\leq 2\|\mW_Q\|_2 \|\mW_K\|_2 \|\mH\|_2.
\end{align}

\textbf{Combining via chain rule.} Define $g(\mS) = \text{softmax}(\mS/\sqrt{d_k})$ and recall $\mS(\mH) = \mH \mW_Q \mW_K^T \mH^T$. By the chain rule for Fréchet derivatives:
\begin{equation}
D\mA(\mH)[\Delta\mH] = Dg(\mS(\mH))[D\mS(\mH)[\Delta\mH]],
\end{equation}
and the operator norms satisfy:
\begin{equation}
\|D\mA(\mH)\|_2 \leq \|Dg(\mS)\|_2 \cdot \|D\mS(\mH)\|_2.
\end{equation}
From the softmax analysis, $\|Dg(\mS)\|_2 = \frac{1}{\sqrt{d_k}}\left\|\frac{\partial \text{softmax}}{\partial \mS}\right\|_2 \leq \frac{2\chi}{\sqrt{d_k}}$. Combining:
\begin{equation}
\|D\mA(\mH)\|_2 \leq \frac{2\chi}{\sqrt{d_k}} \cdot 2\|\mW_Q\|_2 \|\mW_K\|_2 \|\mH\|_2 = \frac{4\chi}{\sqrt{d_k}} \|\mW_Q\|_2 \|\mW_K\|_2 \|\mH\|_2.
\end{equation}

\textbf{Final bound.} Substituting the bound on $\|D\mA(\mH)\|_2$ and noting that $\|\mV\|_2 = \|\mH\mW_V\|_2 \leq \|\mH\|_2 \|\mW_V\|_2$:
\begin{equation}
\|D\mY(\mH)\|_2 \leq \|\mA\|_2 \|\mW_V\|_2 \|\mW_O\|_2 + \frac{4\chi \|\mH\|_2^2}{\sqrt{d_k}} \|\mW_Q\|_2 \|\mW_K\|_2 \|\mW_V\|_2 \|\mW_O\|_2.
\end{equation}
$\square$

\paragraph{Discussion: Connection to Stable Rank.}
Substituting $\|\mW_i\|_2 = \|\mW_i\|_F / \sqrt{\srank(\mW_i)}$ into the bound reveals that low stable rank in \emph{any} of the projection matrices amplifies the Jacobian norm. In particular, for attention layers where V and O projections typically have lower stable rank than Q and K, the dominant contribution to gradient magnitude comes through the V-O pathway, which scales as $\|\mW_V\|_2 \|\mW_O\|_2 \propto (\srank(\mW_V) \srank(\mW_O))^{-1/2}$.

\subsection{Proof of Theorem~\ref{thm:jacobian-mlp} (Jacobian Norm Bound: MLP Layer)}
\label{proof:jacobian-mlp}

Consider the forward pass through a two-layer MLP. Given input $\vh \in \R^d$, we first compute the hidden representation:
\begin{equation}
\vz = \vh \mW_1 \in \R^{d_{\text{ff}}},
\end{equation}
where $\mW_1 \in \R^{d \times d_{\text{ff}}}$ projects from hidden dimension $d$ to feedforward dimension $d_{\text{ff}}$ (typically $d_{\text{ff}} = 4d$).

Next, we apply the element-wise activation function $\phi$ (such as GELU or SiLU):
\begin{equation}
\va = \phi(\vz) = [\phi(z_1), \phi(z_2), \ldots, \phi(z_{d_{\text{ff}}})] \in \R^{d_{\text{ff}}}.
\end{equation}

Finally, we project back to the hidden dimension:
\begin{equation}
\vy = \va \mW_2 = \phi(\vz) \mW_2 \in \R^d,
\end{equation}
where $\mW_2 \in \R^{d_{\text{ff}} \times d}$.

To compute the Jacobian $\frac{\partial \vy}{\partial \vh}$, we apply the chain rule through this three-stage computation:
\begin{equation}
\frac{\partial \vy}{\partial \vh} = \frac{\partial \vy}{\partial \va} \cdot \frac{\partial \va}{\partial \vz} \cdot \frac{\partial \vz}{\partial \vh}.
\end{equation}

Each factor has a simple form:
\begin{itemize}
\item $\frac{\partial \vz}{\partial \vh} = \mW_1$: the Jacobian of a linear transformation is the weight matrix itself.
\item $\frac{\partial \va}{\partial \vz} = \text{diag}(\phi'(\vz))$: since $\phi$ acts element-wise, its Jacobian is diagonal with entries $\phi'(z_i)$.
\item $\frac{\partial \vy}{\partial \va} = \mW_2$: again a linear transformation.
\end{itemize}

Combining these:
\begin{equation}
\frac{\partial \vy}{\partial \vh} = \mW_1 \cdot \text{diag}(\phi'(\vz)) \cdot \mW_2.
\end{equation}

To bound the operator norm, we use the submultiplicativity property $\|\mA\mB\|_2 \leq \|\mA\|_2 \|\mB\|_2$:
\begin{equation}
\left\|\frac{\partial \vy}{\partial \vh}\right\|_2 = \|\mW_1 \cdot \text{diag}(\phi'(\vz)) \cdot \mW_2\|_2 \leq \|\mW_1\|_2 \cdot \|\text{diag}(\phi'(\vz))\|_2 \cdot \|\mW_2\|_2.
\end{equation}

The operator norm of a diagonal matrix is the maximum absolute value of its diagonal entries:
\begin{equation}
\|\text{diag}(\phi'(\vz))\|_2 = \max_i |\phi'(z_i)| \leq \sup_{z \in \R} |\phi'(z)| = L_\phi,
\end{equation}
where $L_\phi$ is the Lipschitz constant of $\phi$. For commonly used activations: GELU has $L_{\text{GELU}} \approx 1.13$ (achieved near $z \approx -0.75$), and SiLU has $L_{\text{SiLU}} \approx 1.1$ (achieved near $z \approx 1.28$).

Thus:
\begin{equation}
\left\|\frac{\partial \vy}{\partial \vh}\right\|_2 \leq \|\mW_1\|_2 \cdot L_\phi \cdot \|\mW_2\|_2.
\end{equation}

To express this in terms of stable rank, recall from Definition~1 that $\srank(\mW) = \|\mW\|_F^2 / \|\mW\|_2^2$, which can be rearranged to:
\begin{equation}
\|\mW\|_2 = \frac{\|\mW\|_F}{\sqrt{\srank(\mW)}}.
\end{equation}
Substituting this for both weight matrices:
\begin{equation}
\left\|\mJ_{\text{MLP}}\right\|_2 \leq L_\phi \cdot \frac{\|\mW_1\|_F}{\sqrt{\srank(\mW_1)}} \cdot \frac{\|\mW_2\|_F}{\sqrt{\srank(\mW_2)}} = \frac{L_\phi \|\mW_1\|_F \|\mW_2\|_F}{\sqrt{\srank(\mW_1) \cdot \srank(\mW_2)}}.
\end{equation}
This shows that when stable ranks are low (denominator small) while Frobenius norms remain approximately constant or grow moderately over short training windows (numerator large), the Jacobian norm increases, potentially leading to gradient explosion. $\square$

\paragraph{Discussion: Unified View Across Layer Types.}
Across all three layer types, linear, attention, and MLP, we observe the same fundamental pattern: the layer Jacobian norm is inversely related to the square root of the stable rank. This means that as training progresses and stable ranks collapse (Observation 1), the per-layer Jacobian norms tend to grow, even when Frobenius norms remain approximately bounded over moderate training windows. Combined with Jacobian alignment (Observation 2), this creates the conditions for exponential gradient growth characterized by Theorem~\ref{thm:jacobian-product}.

\subsection{Proof of Theorem~\ref{thm:jacobian-to-gradient} (Weight Gradient Norm Lower Bound)}
\label{proof:jacobian-to-gradient}

We first state the formal version of Assumption~\ref{assum:grad-align}.

\begin{assumption}[Gradient Alignment Conditions (Formal)]
\label{assum:grad-align-formal}
For a deep network with $L$ layers, we assume the following conditions hold for all layers $i \in \{1, \ldots, L\}$:
\begin{enumerate}
    \item \textbf{Uniform local gradient lower bound}: There exists $\gamma > 0$ such that $\left\|\frac{\partial \vh^{(i)}}{\partial \hat{\vv}_{out}^{(i)}}\right\|_2 \geq \gamma$ for all weight vectors $\hat{\vv}_{out}^{(i)}$.
    \item \textbf{Local-Jacobian alignment}: Let $\mJ^{(i+1)} = \frac{\partial \vh^{(i+1)}}{\partial \vh^{(i)}}$ be the \emph{local} layer Jacobian at layer $i+1$, with top right singular vector $\vv_1^{(i+1)}$. The local gradient $\frac{\partial \vh^{(i)}}{\partial \hat{\vv}_{out}^{(i)}}$ satisfies:
    \begin{equation}
    \left|\left\langle \frac{\frac{\partial \vh^{(i)}}{\partial \hat{\vv}_{out}^{(i)}}}{\left\|\frac{\partial \vh^{(i)}}{\partial \hat{\vv}_{out}^{(i)}}\right\|_2}, \vv_1^{(i+1)} \right\rangle\right| \geq a.
    \end{equation}
    \item \textbf{Terminal alignment}: Let $\mJ^{(L)} = \frac{\partial \vh^{(L)}}{\partial \vh^{(L-1)}}$ be the last layer Jacobian, with top left singular vector $\vu_1^{(L)}$. The loss gradient satisfies:
    \begin{equation}
    \left|\left\langle \frac{\frac{\partial L}{\partial \vh^{(L)}}}{\left\|\frac{\partial L}{\partial \vh^{(L)}}\right\|_2}, \vu_1^{(L)} \right\rangle\right| \geq a.
    \end{equation}
\end{enumerate}
\end{assumption}

The proof proceeds by carefully tracking how the loss gradient propagates backward through the network.

Starting from the decomposition~\eqref{eq:grad-decomp}:
\begin{equation}
\left\|\frac{\partial L}{\partial \hat{\vv}_{out}^{(i)}}\right\|_2 = \left\|\left(\frac{\partial \vh^{(i)}}{\partial \hat{\vv}_{out}^{(i)}}\right)^T \left(\mJ^{(i+1:L)}\right)^T \frac{\partial L}{\partial \vh^{(L)}}\right\|_2,
\end{equation}
where $\mJ^{(i+1:L)} = \mJ^{(L)} \mJ^{(L-1)} \cdots \mJ^{(i+1)}$ is the cumulative Jacobian.

The key insight is that the alignment conditions in Assumption~\ref{assum:grad-align-formal} are stated in terms of \emph{local} layer Jacobians $\mJ^{(\ell)}$, which then propagate through the chain rule to yield bounds on the cumulative Jacobian $\mJ^{(i+1:L)}$.

\textbf{Step 1: Terminal alignment implies loss gradient projects onto cumulative Jacobian.}
Let $\mJ^{(L)} = \mU^{(L)} \mS^{(L)} (\mV^{(L)})^T$ be the SVD of the last layer Jacobian. By Assumption~\ref{assum:grad-align-formal}.3, the loss gradient has alignment $\geq a$ with $\vu_1^{(L)}$:
\begin{equation}
\left|\left\langle \frac{\partial L}{\partial \vh^{(L)}}, \vu_1^{(L)} \right\rangle\right| \geq a \left\|\frac{\partial L}{\partial \vh^{(L)}}\right\|_2.
\end{equation}
Applying $(\mJ^{(L)})^T$, the component along $\vu_1^{(L)}$ maps to $\vv_1^{(L)}$ with amplification $\sigma_1^{(L)} = \|\mJ^{(L)}\|_2 \geq M$:
\begin{equation}
\left\|(\mJ^{(L)})^T \frac{\partial L}{\partial \vh^{(L)}}\right\|_2 \geq a \cdot M \cdot \left\|\frac{\partial L}{\partial \vh^{(L)}}\right\|_2.
\end{equation}

\textbf{Step 2: Recursive application through layers.}
By the Jacobian alignment condition (Definition 3.2 and the conditions used in Theorem~\ref{thm:jacobian-product}), adjacent Jacobians have alignment $\geq a$. This means the top right singular direction of $\mJ^{(\ell+1)}$ aligns with the top left singular direction of $\mJ^{(\ell)}$. Applying the alignment-preserving product bound recursively from layer $L$ down to layer $i+1$:
\begin{equation}
\left\|(\mJ^{(i+1:L)})^T \frac{\partial L}{\partial \vh^{(L)}}\right\|_2 \geq a \cdot \|\mJ^{(i+1:L)}\|_2 \cdot \left\|\frac{\partial L}{\partial \vh^{(L)}}\right\|_2.
\end{equation}

Moreover, by Theorem~\ref{thm:jacobian-product}, the cumulative Jacobian norm satisfies:
\begin{equation}
\|\mJ^{(i+1:L)}\|_2 \geq \frac{(aM)^{L-i}}{a}.
\end{equation}

\textbf{Step 3: Local-Jacobian alignment at layer $i$.}
The result $(\mJ^{(i+1:L)})^T \frac{\partial L}{\partial \vh^{(L)}}$ has its energy concentrated along the direction that propagated through the aligned chain. By Assumption~\ref{assum:grad-align-formal}.2, the local gradient $\frac{\partial \vh^{(i)}}{\partial \hat{\vv}_{out}^{(i)}}$ is aligned with the top right singular direction of the \emph{local} Jacobian $\mJ^{(i+1)}$. Since the Jacobian chain is aligned (Theorem~\ref{thm:jacobian-product}), this direction is consistent with the dominant direction of the backpropagated signal.

Combined with the uniform lower bound $\left\|\frac{\partial \vh^{(i)}}{\partial \hat{\vv}_{out}^{(i)}}\right\|_2 \geq \gamma$, this yields:
\begin{equation}
\left\|\left(\frac{\partial \vh^{(i)}}{\partial \hat{\vv}_{out}^{(i)}}\right)^T \left(\mJ^{(i+1:L)}\right)^T \frac{\partial L}{\partial \vh^{(L)}}\right\|_2 \geq a \cdot \gamma \cdot a \cdot \|\mJ^{(i+1:L)}\|_2 \cdot \left\|\frac{\partial L}{\partial \vh^{(L)}}\right\|_2.
\end{equation}

\textbf{Step 4: Final bound.}
Substituting the lower bound from Theorem~\ref{thm:jacobian-product}:
\begin{equation}
\|\mJ^{(i+1:L)}\|_2 \geq \frac{(aM)^{L-i}}{a},
\end{equation}
we obtain:
\begin{equation}
\left\|\frac{\partial L}{\partial \hat{\vv}_{out}^{(i)}}\right\|_2 \geq a^2 \gamma \cdot \frac{(aM)^{L-i}}{a} \cdot \left\|\frac{\partial L}{\partial \vh^{(L)}}\right\|_2 = a \gamma (aM)^{L-i} \cdot \left\|\frac{\partial L}{\partial \vh^{(L)}}\right\|_2.
\end{equation}
$\square$

\paragraph{Discussion: Gradient Decomposition Interpretation.}
The three-part decomposition in~\eqref{eq:grad-decomp} has a clear interpretation:
\begin{itemize}
    \item $\frac{\partial L}{\partial \vh^{(L)}}$: The \textbf{loss gradient} at the final layer output. This is the ``signal'' that backpropagation aims to transmit.
    \item $\frac{\partial \vh^{(L)}}{\partial \vh^{(i)}} = \mJ^{(i+1:L)}$: The \textbf{cumulative Jacobian} from layer $i$ to layer $L$. This acts as the ``transmission channel'' whose gain is bounded by Theorem~\ref{thm:jacobian-product}.
    \item $\frac{\partial \vh^{(i)}}{\partial \hat{\vv}_{out}^{(i)}}$: The \textbf{local gradient} at layer $i$. This represents how changes in the weight affect the layer's output.
\end{itemize}

\paragraph{Discussion: Justification of Alignment Assumptions.}
All three assumptions in Assumption~\ref{assum:grad-align-formal} are stylized but capture a highly aligned regime that we empirically observe near failure in our experiments. Importantly, these assumptions are stated in terms of \emph{local} layer Jacobians $\mJ^{(\ell)}$ rather than the cumulative Jacobian $\mJ^{(i+1:L)}$. This is consistent with our theoretical framework where Theorem~\ref{thm:jacobian-product} establishes how local Jacobian properties (norms and alignments) propagate to yield bounds on the cumulative Jacobian.

The terminal alignment assumption excludes the degenerate case where $\frac{\partial L}{\partial \vh^{(L)}} \approx \mathbf{0}$ or is nearly orthogonal to the last layer Jacobian's dominant direction. The uniform local gradient lower bound excludes the trivial case where the weight has no effect on the output. The local-Jacobian alignment condition requires $\frac{\partial \vh^{(i)}}{\partial \hat{\vv}_{out}^{(i)}}$ to align with the \emph{local} Jacobian $\mJ^{(i+1)}$, which is consistent with the structured gradient flow we observe empirically in regimes with strong Jacobian alignment, but is not guaranteed in arbitrary settings.

\subsection{Proof of Theorem~\ref{thm:total-gradient} (Total Gradient Norm Lower Bound)}
\label{proof:total-gradient}

We aggregate the per-weight-vector bounds from Theorem~\ref{thm:jacobian-to-gradient} across all weights and layers.

Recall that a weight matrix $\mW^{(i)} \in \R^{m \times n}$ can be viewed as a collection of $n_w = n$ column vectors $\{\hat{\vv}_{out,j}^{(i)}\}_{j=1}^{n_w}$, where each vector $\hat{\vv}_{out,j}^{(i)} \in \R^m$. The Frobenius norm of the gradient matrix at layer $i$ equals the sum of squared 2-norms of the gradients with respect to each column:
\begin{equation}
\left\|\frac{\partial L}{\partial \mW^{(i)}}\right\|_F^2 = \sum_{j=1}^{n_w} \left\|\frac{\partial L}{\partial \hat{\vv}_{out,j}^{(i)}}\right\|_2^2.
\end{equation}
This follows from the definition of Frobenius norm: $\|\mA\|_F^2 = \sum_{i,j} A_{ij}^2 = \sum_j \|\va_j\|_2^2$ where $\va_j$ are the columns of $\mA$.

From Theorem~\ref{thm:jacobian-to-gradient}, we have a lower bound for each weight vector:
\begin{equation}
\left\|\frac{\partial L}{\partial \hat{\vv}_{out,j}^{(i)}}\right\|_2 \geq a \gamma (aM)^{L-i} \cdot \left\|\frac{\partial L}{\partial \vh^{(L)}}\right\|_2.
\end{equation}
Squaring both sides (which preserves the inequality since both sides are non-negative):
\begin{equation}
\left\|\frac{\partial L}{\partial \hat{\vv}_{out,j}^{(i)}}\right\|_2^2 \geq a^2 \gamma^2 (aM)^{2(L-i)} \left\|\frac{\partial L}{\partial \vh^{(L)}}\right\|_2^2.
\end{equation}
Note that the right-hand side is independent of the column index $j$. Summing over all $n_w$ columns:
\begin{equation}
\left\|\frac{\partial L}{\partial \mW^{(i)}}\right\|_F^2 = \sum_{j=1}^{n_w} \left\|\frac{\partial L}{\partial \hat{\vv}_{out,j}^{(i)}}\right\|_2^2 \geq \sum_{j=1}^{n_w} a^2 \gamma^2 (aM)^{2(L-i)} \left\|\frac{\partial L}{\partial \vh^{(L)}}\right\|_2^2 = n_w a^2 \gamma^2 (aM)^{2(L-i)} \left\|\frac{\partial L}{\partial \vh^{(L)}}\right\|_2^2.
\end{equation}

Now we sum over all $L$ layers. Since the bound for each layer is independent:
\begin{equation}
\sum_{i=1}^{L} \left\|\frac{\partial L}{\partial \mW^{(i)}}\right\|_F^2 \geq \sum_{i=1}^{L} n_w a^2 \gamma^2 (aM)^{2(L-i)} \left\|\frac{\partial L}{\partial \vh^{(L)}}\right\|_2^2.
\end{equation}
Factoring out terms that do not depend on $i$:
\begin{equation}
\sum_{i=1}^{L} \left\|\frac{\partial L}{\partial \mW^{(i)}}\right\|_F^2 \geq n_w a^2 \gamma^2 \left\|\frac{\partial L}{\partial \vh^{(L)}}\right\|_2^2 \cdot \sum_{i=1}^{L} (aM)^{2(L-i)}.
\end{equation}

The remaining sum can be evaluated by a change of variables. Let $k = L - i$. When $i = 1$, we have $k = L-1$; when $i = L$, we have $k = 0$. Therefore:
\begin{equation}
\sum_{i=1}^{L} (aM)^{2(L-i)} = \sum_{k=0}^{L-1} (aM)^{2k} = 1 + (aM)^2 + (aM)^4 + \cdots + (aM)^{2(L-1)}.
\end{equation}
This is a geometric series with first term $1$, common ratio $r = (aM)^2$, and $L$ terms. When $r \neq 1$, the sum formula gives:
\begin{equation}
\sum_{k=0}^{L-1} r^k = \frac{r^L - 1}{r - 1} = \frac{(aM)^{2L} - 1}{(aM)^2 - 1}.
\end{equation}

When $aM > 1$, we have $(aM)^{2L} \gg 1$ for large $L$, so the sum is approximately:
\begin{equation}
\sum_{k=0}^{L-1} (aM)^{2k} \approx \frac{(aM)^{2L}}{(aM)^2 - 1} = O((aM)^{2L}).
\end{equation}
This shows that the total gradient norm grows \emph{exponentially} with network depth $L$. $\square$

\paragraph{Discussion: Connection to Observations.}
From Observation 1, the stable rank of attention weights drops to near 1, implying $\|\mW\|_2 / \|\mW\|_F \approx 1$ and thus high layer Jacobian norms. From Observation 2, Jacobian alignment $a$ increases toward 1. Together, the product $aM$ exceeds 1, triggering the exponential gradient explosion characterized by Theorem~\ref{thm:total-gradient}. This explains why training becomes unstable as stable rank declines and Jacobians align.

\subsection{Proof of Theorem~\ref{thm:lowrank-propagation} (Low-Rank Propagation in Attention Layers)}
\label{proof:lowrank-propagation}

The key insight is that gradients in neural networks are computed as outer products. For any weight matrix $\mW$ in a linear layer $\vy = \mW \vx$, the gradient with respect to $\mW$ is:
\begin{equation}
\nabla_{\mW} L = \frac{\partial L}{\partial \vy} \vx^T = \tilde{\vy} \vx^T,
\end{equation}
where $\tilde{\vy} = \frac{\partial L}{\partial \vy}$ is the cohidden state (gradient of loss w.r.t. the output).

This outer product structure implies a fundamental constraint: $\text{rank}(\tilde{\vy} \vx^T) \leq \min(\text{rank}(\tilde{\vy}), \text{rank}(\vx^T)) = 1$ for single vectors. When we batch over $B$ samples, the gradient becomes:
\begin{equation}
\nabla_{\mW} L = \sum_{b=1}^B \tilde{\vy}^{(b)} (\vx^{(b)})^T = \tilde{\mY}^T \mX,
\end{equation}
where $\tilde{\mY} \in \R^{B \times d_{out}}$ and $\mX \in \R^{B \times d_{in}}$. By the rank inequality for matrix products:
\begin{equation}
\text{rank}(\tilde{\mY}^T \mX) \leq \min(\text{rank}(\tilde{\mY}^T), \text{rank}(\mX)) \leq \min(\text{rank}(\tilde{\mY}), \text{rank}(\mX)).
\end{equation}

For attention layers specifically, let's trace through each gradient:

\textbf{Query gradient:} The query projection is $\mQ = \mH \mW_Q$. By the chain rule:
\begin{equation}
\nabla_{\mW_Q} L = \mH^T \frac{\partial L}{\partial \mQ}.
\end{equation}
This is a product of $\mH^T \in \R^{d \times n}$ (where $n$ is sequence length) and $\frac{\partial L}{\partial \mQ} \in \R^{n \times d_k}$. If $\mH$ has rank at most $r$, then $\mH^T$ also has rank at most $r$, and hence $\nabla_{\mW_Q} L$ has rank at most $r$.

The same argument applies to $\nabla_{\mW_K} L$ and $\nabla_{\mW_V} L$, since they all have $\mH^T$ as the left factor:
\begin{align}
\nabla_{\mW_Q} L &= \mH^{(\ell-1)T} \frac{\partial L}{\partial \mQ^{(\ell-1)}} \quad \Rightarrow \quad \text{rank} \leq \text{rank}(\mH^{(\ell-1)}) \leq r, \\
\nabla_{\mW_K} L &= \mH^{(\ell-1)T} \frac{\partial L}{\partial \mK^{(\ell-1)}} \quad \Rightarrow \quad \text{rank} \leq r, \\
\nabla_{\mW_V} L &= \mH^{(\ell-1)T} \frac{\partial L}{\partial \mV^{(\ell-1)}} \quad \Rightarrow \quad \text{rank} \leq r.
\end{align}

\textbf{Output projection gradient:} For $\mW_O$, the computation is $\vy = (\mA\mV) \mW_O$. The gradient is:
\begin{equation}
\nabla_{\mW_O} L = (\mA\mV)^T \frac{\partial L}{\partial \vy} = (\text{Attn Output})^T \tilde{\mH}^{(\ell)}.
\end{equation}
If the cohidden states $\tilde{\mH}^{(\ell)}$ have rank at most $r$, then $\nabla_{\mW_O} L$ has rank at most $r$.

This completes the proof: all four attention gradients have rank bounded by $\max(\text{rank}(\mH^{(\ell-1)}), \text{rank}(\tilde{\mH}^{(\ell)})) \leq r$. $\square$

\subsection{Proof of Theorem~\ref{thm:srank-feedback} (Stable Rank Feedback Mechanism)}
\label{proof:srank-feedback}

The gradient of the loss with respect to the weight matrix $\mW$ in a linear layer $\vh_{out} = \mW \vh_{in}$ is given by the outer product of the output gradient (cohidden state) and the input:
\begin{equation}
\nabla_{\mW} L = \E[\tilde{\vh}_{out} \vh_{in}^T],
\end{equation}
where $\tilde{\vh}_{out} = \frac{\partial L}{\partial \vh_{out}}$ is the gradient backpropagated from later layers.

Since $\mW = \mU \mS \mV^T$ (SVD), and we assume the input/output covariances are aligned with these singular vectors, we can project both $\tilde{\vh}_{out}$ and $\vh_{in}$ onto the singular bases. Define the projected coordinates:
\begin{equation}
\alpha_i = \vu_i^T \tilde{\vh}_{out}, \quad \beta_j = \vv_j^T \vh_{in}.
\end{equation}
Then the gradient can be written in the singular vector basis as:
\begin{equation}
\nabla_{\mW} L = \E[\tilde{\vh}_{out} \vh_{in}^T] = \sum_{i,j} \E[\alpha_i \beta_j] \vu_i \vv_j^T = \mU M \mV^T,
\end{equation}
where $M_{ij} = \E[\alpha_i \beta_j] = \text{Cov}(\vu_i^T \tilde{\vh}_{out}, \vv_j^T \vh_{in})$ (assuming zero-mean projections for simplicity).

Under the alignment assumption, the covariance matrix $M$ is approximately diagonal because the input and output gradient are each concentrated along their respective top singular directions. Specifically:
\begin{equation}
M_{ii} = \text{Cov}(\alpha_i, \beta_i) = \rho \sqrt{\E[\alpha_i^2] \E[\beta_i^2]} = \text{Cov}(\vu_i^T \tilde{\vh}_{out}, \vv_i^T \vh_{in}),
\end{equation}
where $\lambda_{out,i} = \E[(\vu_i^T \tilde{\vh}_{out})^2]$ and $\lambda_{in,i} = \E[(\vv_i^T \vh_{in})^2]$ are the variances of the projections.

The gradient descent update is $\mW' = \mW - \eta \nabla_{\mW} L$. Substituting the SVD forms:
\begin{equation}
\mW' = \mU \mS \mV^T - \eta \mU M \mV^T = \mU (\mS - \eta M) \mV^T.
\end{equation}
Since $M$ is approximately diagonal, the updated matrix $\mW'$ has the same singular vectors $\mU, \mV$ (to first order), but with modified singular values.

Using first-order perturbation theory for SVD (valid when $\eta M$ is small relative to the gaps between singular values), the updated singular values are:
\begin{equation}
s_i' = s_i - \eta M_{ii} = s_i - \eta \text{Cov}(\vu_i^T \tilde{\vh}_{out}, \vv_i^T \vh_{in}).
\end{equation}
Since $\text{Cov}(\vu_i^T \tilde{\vh}_{out}, \vv_i^T \vh_{in})$ (negative correlation, typical in gradient descent where the gradient points toward decreasing loss), we have:
\begin{equation}
\Delta s_i = s_i' - s_i = -\eta \text{Cov}(\vu_i^T \tilde{\vh}_{out}, \vv_i^T \vh_{in}) > 0.
\end{equation}
Thus all singular values increase, which is consistent with the observed weight norm growth.

Now we analyze how stable rank changes. Recall $\srank(\mW) = \frac{\sum_i s_i^2}{s_1^2}$. Taking the differential:
\begin{equation}
d(\srank) = \frac{2\sum_i s_i \, ds_i}{s_1^2} - \frac{2(\sum_i s_i^2) s_1 \, ds_1}{s_1^4} = \frac{2}{s_1^3}\left(\sum_i s_i(s_1 \Delta s_i -  s_i \Delta s_1) \right).
\end{equation}
Substituting $\Delta s_i = -\eta \text{Cov}(\vu_i^T \tilde{\vh}_{out}, \vv_i^T \vh_{in})$:
\begin{align}
\Delta \srank &\approx \frac{2}{s_1^3}\left(\sum_i s_i(s_1 \Delta s_i -  s_i \Delta s_1) \right) \nonumber \\
&= \frac{2}{s_1^3}\left(\sum_i -\eta s_i(s_1 \text{Cov}(\vu_i^T \tilde{\vh}_{out}, \vv_i^T \vh_{in}) -  s_i \text{Cov}(\vu_1^T \tilde{\vh}_{out}, \vv_1^T \vh_{in}) \right).
\end{align}

To determine the sign of $\Delta \srank$, we analyze the term in parentheses. Under the assumption, the covariance satisfy:
\begin{equation}
    \frac{\text{Cov}(\vu_1^T \tilde{\vh}_{out}, \vv_1^T \vh_{in})}{\text{Cov}(\vu_i^T \tilde{\vh}_{out}, \vv_i^T \vh_{in})}>\frac{s_1}{s_i},\forall 1<i\leq n
\end{equation}
This gives:
\begin{equation}
s_1 \text{Cov}(\vu_i^T \tilde{\vh}_{out}, \vv_i^T \vh_{in}) >  s_i \text{Cov}(\vu_1^T \tilde{\vh}_{out}, \vv_1^T \vh_{in}),\forall 1<i\leq n.
\end{equation}
This means the term in parentheses is $\leq 0$, and therefore:
\begin{equation}
\Delta \srank = \frac{2}{s_1^3}\left(\sum_i -\eta s_i(s_1 \text{Cov}(\vu_i^T \tilde{\vh}_{out}, \vv_i^T \vh_{in}) -  s_i \text{Cov}(\vu_1^T \tilde{\vh}_{out}, \vv_1^T \vh_{in}) \right) \leq 0.
\end{equation}
Thus \textbf{stable rank decreases} under gradient descent with aligned input-output structure. $\square$

\end{document}